\documentclass[twoside]{article}

\usepackage[preprint]{aistats2027}
\usepackage[round,authoryear]{natbib}
\usepackage[utf8]{inputenc}
\usepackage[T1]{fontenc}
\usepackage{lmodern}
\usepackage{amsmath,amssymb,amsthm,mathtools}
\usepackage{microtype}
\usepackage{booktabs}
\usepackage{graphicx}
\usepackage{tabularx}
\usepackage{enumitem}
\usepackage{url}
\usepackage[table]{xcolor}
\usepackage{etoolbox}
\usepackage{placeins}
\usepackage{float}
\floatstyle{ruled}
\newfloat{algorithm}{tbp}{loa}
\floatname{algorithm}{Algorithm}
\definecolor{tablehead}{HTML}{EDF0F3}
\definecolor{exchead}{HTML}{EEEAF6}
\definecolor{cdfhead}{HTML}{FBEDE4}
\definecolor{recoveryrow}{HTML}{EAF3ED}
\AtBeginEnvironment{table}{\renewcommand{\arraystretch}{1.08}}
\AtBeginEnvironment{table*}{\renewcommand{\arraystretch}{1.08}}
\usepackage{hyperref}
\definecolor{linkblue}{HTML}{0057B8}
\hypersetup{colorlinks=true,allcolors=linkblue}

\theoremstyle{plain}
\newtheorem{theorem}{Theorem}
\newtheorem{proposition}{Proposition}
\newtheorem{corollary}{Corollary}
\newtheorem{lemma}{Lemma}
\theoremstyle{definition}
\newtheorem{assumption}{Assumption}
\newtheorem{remark}{Remark}

\newcommand{\E}{\mathbb E}
\newcommand{\PP}{\mathbb P}
\newcommand{\KL}{\mathrm{KL}}
\newcommand{\kl}{\mathrm{kl}}
\newcommand{\1}{\mathbb 1}
\newcommand{\calD}{\mathcal D}
\newcommand{\keycell}[1]{{\bfseries\boldmath #1}}

\begin{document}
\runningauthor{Pauline Bourigault, Emmanuelle Bourigault}
\twocolumn[
\aistatstitle{Coverage Cliffs in Off-Policy Quantile Certification}
\aistatsauthor{Pauline Bourigault \And Emmanuelle Bourigault}
\aistatsaddress{Imperial College London \And Huawei Noah's Ark Lab}
]

\begin{abstract}
We study when logged contextual-bandit data can certify an upper tail quantile for a policy selected from the same sample. We construct a PAC-Bayes Bernstein lower-confidence envelope for the target cumulative distribution function (CDF), simultaneous over a predeclared threshold grid and all data-dependent posteriors. The resulting procedure, TailCert, uses the envelope's first crossing of level $p$ to return an atom-safe upper confidence bound on the selected posterior's pooled $p$-quantile (or inconclusive). It can also be used to select a policy under a certified tail constraint. The concentration ingredients are standard; the main theoretical contribution is a matched finite-class certification boundary comparing full logs with their joint weighted CDF indicators at a fixed threshold. On an explicit rare-action class, the required sample size grows linearly with the overlap factor $W$, logarithmically with policy-search and confidence complexity, and quadratically with the inverse CDF safety margin; the upper and lower bounds agree up to constants in the stated regime. On this class, the CDF reduction hides which actions exceeded the budget; retaining action counts restores the full-log rate. Simulations verify the predicted search-complexity shift, winner's-curse failure of naive reuse, and the power cost of splitting.
\par\smallskip\noindent\textbf{Code:} \url{https://github.com/paulinebourigault/coverage-cliffs}
\end{abstract}
\section{Introduction}
Before deploying a policy, that is a rule for choosing actions, such as which language model answers a request, an operator may require evidence that at least $99\%$ of requests cost at most a budget $\tau$. We ask how the information retained from historical logs affects this certification when the policy is selected from independently fixed candidates using the same data. On an explicit model, keeping only weighted in-budget indicators increases the required sample size. We identify the lost information and summaries that restore access to an established certificate without new outcomes.

Information loss is already visible with two uniformly logged actions and binary outcomes (Figure~\ref{fig:two_action_recovery}). A weighted cumulative distribution function (CDF) indicator records the action when the outcome stays within budget, but becomes zero for either action when the budget is exceeded. It therefore hides which action caused each rare failure. The original log contains those labels. As the desired percentile increases, losing them can be much more costly than any constant-factor slack in a confidence inequality.

For a fixed policy with outcome CDF $F$ and $p$-quantile $Q_p$, $p\in(0,1)$, the constraint $Q_p\le\tau$ is equivalent to both $F(\tau)\ge p$ and $G(\tau):=\PP(Y>\tau)\le1-p$. For $n$ logged outcomes $Y_i$ with importance weights $w_i$, define
\[
 \widehat F=\frac1n\sum_iw_i\1\{Y_i\le\tau\},\qquad
 \widehat G=\frac1n\sum_iw_i\1\{Y_i>\tau\},
\]
where $\1$ is the indicator function. Their sum is the average weight $\bar w=n^{-1}\sum_iw_i$, which need not equal one. The complementary-CDF estimator $\widehat G$ and its favorable upper-tail variance are established results of \citet{huang2022offarxiv}, who also give empirical-Bernstein bounds and combine the two estimators. Finite-list validity after selection follows by a union bound over such intervals. These estimator guarantees leave a distinct question: how much certification power is lost when the log is reduced to the joint CDF indicators, and can \emph{any} rule using only that reduction avoid the cost? We also ask which additional summaries restore access to a full-log certificate.

We use \emph{coverage cliffs} informally for limits on reliably certifying a safe policy from the available logged information. The results quantify sample requirements at controlled error; insufficient evidence can leave a valid procedure inconclusive. On an explicit rare-action family, we compare the full log with its joint weighted in-budget indicators across $K$ candidates, isolating information loss from the choice of confidence inequality. Write $u=1-p$, let $W$ bound the importance weights, let $\gamma$ be the safe policy's slack in the tail constraint, and let $\delta$ be the false-certification tolerance. At fixed power, the full-data requirement is order $Wu\log(K/\delta)/\gamma^2$, while the joint-CDF reduction requires order $Wp\log(K/\delta)/\gamma^2$. At relative margin $\gamma=cu$ with fixed $c\in(0,1/2]$, these grow as $1/u$ and $1/u^2$. The statement concerns all rules using the specified reduction, including joint processing across candidates; it does not cover methods retaining weights, action counts, or other normalization information.

The practical implication is a storage rule: \textbf{retain the sum of weights and the sum of squared weights alongside CDF moment summaries.} For fixed candidates, these two totals recover the established exceedance certificate at every stored threshold (Section~\ref{sec:repair}). This can change which policy is certified without collecting new outcomes. TailCert additionally certifies the pooled quantile of a data-selected policy mixture, including when outcomes have point masses.

\paragraph{Contributions.}
\begin{enumerate}[leftmargin=1.25em,itemsep=0.15em,topsep=0.15em]
\item \textbf{An information limit for certification.} We prove matching sample requirements for full logs and joint weighted CDF indicators on the same rare-action family (Theorem~\ref{thm:boundary}), including confidence and search costs. In the stated regime, the order-$p/(1-p)$ penalty applies to every rule using only that reduction.
\item \textbf{A repair without new outcomes.} We prove that action counts restore sufficient information in the binary model (Proposition~\ref{prop:recovery}). For general contextual weights and fixed candidates, two weight totals alongside CDF moments exactly recover the established empirical-Bernstein exceedance certificate (Section~\ref{sec:repair}).
\item \textbf{When this changes deployment decisions.} Paired-log experiments reproduce the predicted loss and recovery (Figure~\ref{fig:information}). The search cost persists across $125$ routers sharing five actions. On RouterBench and a $98{,}000$-generation reasoning study, these certification gaps determine whether useful routers can be certified.
\end{enumerate}

\section{Setting}\label{sec:setup}
Let $\calD_n=\{(X_i,A_i,Y_i)\}_{i=1}^n$ contain $n$ i.i.d.\ logged requests: context $X_i$, action $A_i$ drawn from a known behavior policy $\pi_0$, and loss $Y_i\in[0,B]$ for a known $B>0$. Our main results concern $K$ candidate policies $\pi_1,\ldots,\pi_K$, fixed independently of the certification log, possibly using separate development data. The deployed candidate may nevertheless be selected using the certification log. We assume that the context distribution and the conditional outcome law given context and action remain unchanged when the policy changes. Candidate $j$ has weights $w_j(x,a)=\pi_j(a\mid x)/\pi_0(a\mid x)$ and support contained in the logger's support. The weight corrects how often the logger takes an action relative to the candidate; large weights mean that relevant actions are poorly represented in the log.

\begin{assumption}[Bounded overlap]\label{ass:overlap}
A known finite constant $W$ satisfies $0\le w_j(x,a)\le W$ for all candidates and all $(x,a)$ in the logger's support.
\end{assumption}

A tail budget and a confidence level play different roles. A $99$th-percentile constraint permits at most $1\%$ of future requests to exceed the budget; confidence $1-\delta$ controls the chance that inference from the random log incorrectly certifies the selected policy. An inconclusive result means the procedure has not established the constraint; it does not by itself show that the policy is unsafe.

Write $\PP_{\pi_j}$ and $\E_{\pi_j}$ for probability and expectation when actions follow $\pi_j$. For $p\in(0,1)$ and a threshold $t$, define
\begin{align*}
 G_j(t)&=\PP_{\pi_j}(Y>t),\\
 \widehat G_j(t)&=\frac1n\sum_{i=1}^n w_j\big(X_i,A_i \big)\1\big\{Y_i>t\big\},\\
 Q_{p,j}&=\inf\big\{t:G_j(t)\le1-p\big\}.
\end{align*}
The corresponding CDF and estimate are $F_j(t)=1-G_j(t)$ and $\widehat F_j(t)=n^{-1}\sum_iw_j(X_i,A_i)\1\{Y_i\le t\}$. Importance weighting gives $G_j(t)=\E_{\pi_0}[w_j(X,A)\1\{Y>t\}]$, so $\widehat G_j(t)$ is unbiased for each fixed candidate $j$. Selection is handled by simultaneous confidence bounds, rather than an unbiasedness claim for the selected estimate. The deployment constraint is $G_j(\tau)\le1-p$, equivalently $Q_{p,j}\le\tau$, for a budget $\tau$ fixed before certification. This constraint limits exceedance frequency; it does not further restrict the magnitude of outcomes above $\tau$. If several thresholds are reported, they belong to a finite predeclared grid $\mathcal T\subset[0,B)$ and inference accounts for that multiplicity. The bound $Q_{p,j}\le B$ is always available. Appendix~\ref{app:posterior_certificate} extends the notation and certificate to a data-selected posterior over an independently fixed policy family.

\section{Certification and information loss}\label{sec:boundary}

Fix a threshold $\tau$, write $Z=\1\{Y>\tau\}$ and $u=1-p$, and call candidate $j$ safe when $G_j:=\E_{\pi_j}Z\le u$. A selector-certifier $\Psi\in\{\bot,1,\ldots,K\}$, with $\bot$ denoting inconclusive, is $\delta$-valid on a model class if its probability of outputting an unsafe candidate is at most $\delta$ under every model. We compare procedures using different reductions of the \emph{same} log, while requiring power $P_j\{\Psi=j\}\ge1-\beta$ under a model $P_j$ with uniquely safe candidate $j$; $\beta$ is the tolerated probability of failing to select and certify it. Validity controls wrong certifications; power requires useful decisions and rules out always returning inconclusive. Throughout this section, set
\[
 \delta_*=\min\{\delta,\beta/2\},\qquad
 L_* = \log(2K/\delta_*).
\]

\paragraph{Rare-action family.}
There is one context and $K$ target actions, each logged with probability $1/W$, where $W\ge K\ge2$; a dummy action receives the remaining probability. Candidate $j$ always chooses action $j$, so $w_j\in\{0,W\}$. Outcomes are binary, $Y\in\{0,B\}$, with $\tau\in[0,B)$. Under $P_0$, all target actions have tail probability $u+\gamma$. Under $P_j$, only action $j$ changes, to $u-\gamma$. The dummy outcome law is fixed across models.

A useful procedure must distinguish a log with no safe candidate from one with a safe candidate, then identify which candidate is safe. These requirements produce the confidence and search costs, respectively. We ask how reducing the log changes the data needed for both. Here $\kl$ denotes Bernoulli relative entropy.

\begin{figure}[t]
\centering
\begingroup
\small
\setlength{\tabcolsep}{5pt}
\renewcommand{\arraystretch}{1.12}
\begin{tabular}{@{}lcccc@{}}
\toprule
Observation $i$ & 1 & 2 & 3 & 4\\
Tail flag $Z_i$ & 0 & 1 & 0 & 1\\
\midrule
Action in $\calD$  & 1 & \textbf{1} & 2 & \textbf{1}\\
Action in $\calD'$ & 1 & \textbf{2} & 2 & \textbf{2}\\
\bottomrule
\end{tabular}
\par\smallskip
$\Downarrow\quad\text{retain joint CDF indicators}\quad\Downarrow$
\par\smallskip
\fcolorbox{black!35}{black!4}{%
\parbox{0.87\linewidth}{\centering
Same sequence for both logs:\par\smallskip
$(2,0),\ (0,0),\ (0,2),\ (0,0)$}}
\par\smallskip
\textbf{Recover counts by retaining $N$ as well:}
\par\smallskip
\begin{tabular}{@{}lccc@{}}
\toprule
Log & Actions $N$ & Successes $S$ & Failures $N-S$\\
\midrule
$\calD$  & $(3,1)$ & $(1,1)$ & $(2,0)$\\
$\calD'$ & $(1,3)$ & $(1,1)$ & $(0,2)$\\
\bottomrule
\end{tabular}
\endgroup
\caption{\textbf{The same CDF data can hide opposite failure allocations.} Bold entries mark the erased action labels. The joint CDF sequence preserves the total number of failures but not their allocation to actions. Adding action counts recovers that allocation as counts. This is a schematic example, not experimental evidence.}
\label{fig:two_action_recovery}
\end{figure}

\begin{theorem}[Certification cost before and after CDF reduction]\label{thm:boundary}
Fix $p\in[1/2,1)$, $W\ge K\ge2$, $0<\gamma\le u/2$, and $0<\delta<1-\beta<1$. On the rare-action family $\{P_0,\ldots,P_K\}$, consider $\delta$-valid rules with $P_j\{\Psi=j\}\ge1-\beta$ for every $j$. Put
\[
 H=\max\!\left\{\begin{aligned}
 &\kl(1-\beta \,\| \, \delta),\\
 &(1-\beta)\log K-\log2
 \end{aligned}\right\}.
\]
\emph{(a) Full logs.} Every such rule requires
\begin{equation}\label{eq:lower}
 n\ge\frac{WuH}{16\gamma^2}.
\end{equation}
Certifying when the empirical-Bernstein exceedance upper bound is at most $u$ attains validity and the required output power whenever
\begin{equation}\label{eq:upper}
 n\ge\frac{32WuL_*}{\gamma^2}.
\end{equation}
\emph{(b) Joint CDF indicators.} If a rule observes only
\[
 T_i=\bigl(W\1\{A_i=j,Y_i\le\tau\}\bigr)_{j=1}^K,
 \qquad i=1,\ldots,n,
\]
then it requires
\begin{equation}\label{eq:cdflower}
 n\ge\frac{WpH}{8\gamma^2}.
\end{equation}
Certifying when the empirical-Bernstein CDF lower bound is at least $p$ attains validity and the required output power when $n\ge48WpL_*/\gamma^2$.

Both sufficient procedures use candidate error level $\delta_*/K$ and output any certified candidate. The sufficient bound in (a) also holds without $W\ge K$ for any independently fixed candidate list satisfying Assumption~\ref{ass:overlap}, whenever some candidate $s$ is uniquely safe with $G_s\le u-\gamma$.
\end{theorem}

The factor $W$ reflects how rarely the logger observes each target action. The additional difference is $u$ versus $p$: a full log attributes rare failures to actions, whereas the CDF reduction must identify the safe action through a small change among common successes. For fixed $0<\beta\le1/4$ and $0<\delta\le1/4$, the bounds match at orders $Wu\log(K/\delta)/\gamma^2$ and $Wp\log(K/\delta)/\gamma^2$, up to constants depending on the fixed power. Thus better processing of the reduced data, even jointly across candidates, cannot remove its order-$p/u$ cost on this family. At relative margin $\gamma=cu$, the two costs grow as $1/u$ and $1/u^2$ as $p\to1$. The achievable bounds use established empirical-Bernstein inference \citep{huang2022offarxiv}; the separation quantifies the cost of the observation restriction. Appendix~\ref{app:boundary} gives both testing proofs, and Appendix~\ref{app:information} the exact information comparison.

\paragraph{A worked two-action example.}
Set $K=W=2$: the logger chooses each action with probability $1/2$, and candidate $j$ always takes action $j$. The CDF vector is $(2,0)$ after an in-budget outcome on action 1, $(0,2)$ after one on action 2, and $(0,0)$ after an exceedance on either action. Thus even the \emph{joint} vector reveals that a failure occurred but hides its action.

Figure~\ref{fig:two_action_recovery} gives two four-observation logs with the same CDF sequence. Each has one observed success per action, $S=(1,1)$, but all failures belong to action 1 in one log and action 2 in the other. Adding the action counts $N$ recovers the failure counts $E=N-S$. These counts need not reconstruct the action at each failed observation; they retain what is needed for inference on the two Bernoulli risks.

The action counts alone have the same distribution under every tail-risk model in this construction. Yet combining them with the CDF successes restores the missing failure counts. For rare tails, the informative failure labels are sparse, whereas CDF inference must distinguish small differences among abundant successes. Theorem~\ref{thm:boundary}(b) quantifies this loss as an order-$p/u$ penalty; Proposition~\ref{prop:recovery} formalizes the repair. Neither effect requires small logging probabilities.

\begin{proposition}[Recovery by action counts]\label{prop:recovery}
In the binary rare-action model, let $N_j=\sum_i\1\{A_i=j\}$ and $S_j=\sum_i\1\{A_i=j,Y_i\le\tau\}$. The pair $(N,S)$ is sufficient for arbitrary target-action tail probabilities. Conditional on $N$, the exceedance counts $N_j-S_j$ are independent binomials. In particular, adding $N$ to the CDF success counts recovers $\widehat G_j=W(N_j-S_j)/n$ and the sufficient log-size bound~\eqref{eq:upper}.
\end{proposition}

The factorization proof and exact information comparison are in Appendix~\ref{app:recovery}. Count recovery also explains the scope of Theorem~\ref{thm:boundary}(b): normalized, control-variate, or model-assisted procedures may use information excluded by its observation model. Even $\widehat F_j(B)=WN_j/n$ reveals the missing counts. The restriction therefore concerns the CDF data at $\tau$, rather than every CDF-based method. A CDF mean and empirical second moment at $\tau$ are functions of that reduction; adding actions, separate weights, other thresholds, or full-log utility scores gives a richer observation.

\paragraph{Extension to shared actions.}
Appendix~\ref{app:tailoverlap} gives a sufficient bound governed by the tail second moment $V_j=\E[w_j^2Z]$, together with a global range term. It also proves necessity for families whose shared tail information is sufficiently separated. The construction includes 125 distinct routers over five uniformly logged actions. This extends the scope beyond private actions, while leaving highly redundant candidate lists outside the necessity claim.

\section{Recovering a usable certificate}\label{sec:cert}\label{sec:repair}
We show how two retained weight totals recover an exceedance certificate exactly from CDF moments. For fixed candidates, the result is the same empirical-Bernstein exceedance bound (Exc-EB) computed from full logs. This certificate, due to \citet{huang2022offarxiv} and unioned across candidates, attains the full-data bound in Theorem~\ref{thm:boundary}(a) and the tail-moment bound in Appendix~\ref{app:tailoverlap}.

Write $w_{ij}=w_j(X_i,A_i)$. In the binary construction, action counts supply both weight totals: $\sum_iw_{ij}=WN_j$ and $\sum_iw_{ij}^2=W^2N_j$. With contextual or stochastic policies the weights vary, so we retain these two totals directly. Subtracting the corresponding CDF moments recovers the exceedance mean and variance needed by Exc-EB.

\paragraph{Computing the recovered certificate.}
Under Assumption~\ref{ass:overlap}, fix $\delta\in(0,1)$, $K$ candidate policies, a nonempty finite threshold grid $\mathcal T$, and $n\ge2$ independently of the certification sample. Suppose a CDF computation retains
\begin{align*}
 S_{F,j}(t)&=\sum_iw_{ij}\1\{Y_i\le t\},\\
 Q_{F,j}(t)&=\sum_iw_{ij}^2\1\{Y_i\le t\}.
\end{align*}
Two additional totals per candidate, $S_{W,j}=\sum_iw_{ij}$ and $Q_{W,j}=\sum_iw_{ij}^2$, recover the exceedance sums exactly:
\begin{align*}
 S_{E,j}(t)&=S_{W,j}-S_{F,j}(t),\\
 Q_{E,j}(t)&=Q_{W,j}-Q_{F,j}(t).
\end{align*}
Consequently, with $\widehat G_j(t)=S_{E,j}(t)/n$ and
\[
 \widehat\sigma_j^2(t)=\frac{Q_{E,j}(t)-S_{E,j}(t)^2/n}{n-1},
\]
the exceedance empirical-Bernstein bound (in the form of \citet[Theorem~4]{maurer2009empirical}, which we use for Exc-EB throughout) is recovered exactly:
\begin{equation}\label{eq:summaryrepair}
\begin{aligned}
 r_j(t)&=\sqrt{\frac{2\widehat\sigma_j^2(t)\ell}{n}}
               +\frac{7W\ell}{3(n-1)},\\
 U_j(t)&=\min\{1,\widehat G_j(t)+r_j(t)\},\\
 \ell&=\log\frac{2K|\mathcal T|}{\delta}.
\end{aligned}
\end{equation}
With probability at least $1-\delta$, $G_j(t)\le U_j(t)$ simultaneously for all $(j,t)$, including a pair selected using these data.

\paragraph{Why the reconstruction is valid.}
The identity $\1\{Y>t\}=1-\1\{Y\le t\}$ gives both sums; since the indicator is binary, they are the first and second sample moments of $w_{ij}\1\{Y_i>t\}$. Apply Lemma~\ref{lem:eb} at error level $\delta/(K|\mathcal T|)$ and take a union bound. Clipping is valid because $G_j(t)\le1$.

\textbf{Two totals preserve access to the same Exc-EB certificate as direct exceedance processing.} They work at every stored threshold and for contextual stochastic policies. This reconstruction does not establish statistical sufficiency outside the binary model, minimal storage, or recovery of the observation sequence needed by general betting procedures. The denominator remains $n$, rather than the observed total weight; fixed-sample validity does not license optional stopping.

\begin{algorithm}[t]
\caption{Recover and certify with Exc-EB}
\label{alg:recovered}
\textbf{Setup.} Under Assumption~\ref{ass:overlap}, fix $K$ candidates, a threshold grid $\mathcal T\ni\tau$, and $n\ge2$ independently of the certification sample; specify $p,\delta\in(0,1)$ and a known weight bound $W$.

\textbf{Inputs.} Retained $S_{F,j}(t),Q_{F,j}(t)$ and totals $S_{W,j},Q_{W,j}$; utility scores $\widehat v_j$.
\begin{enumerate}[leftmargin=*,topsep=3pt,itemsep=2pt,parsep=0pt]
\item For every $(j,t)$, recover $S_{E,j}(t)=S_{W,j}-S_{F,j}(t)$ and $Q_{E,j}(t)=Q_{W,j}-Q_{F,j}(t)$.
\item Form the recovered mean and sample variance; compute $U_j(t)$ using~\eqref{eq:summaryrepair}.
\item Form the certified set $\mathcal C_\tau=\{j:U_j(\tau)\le1-p\}$.
\item If $\mathcal C_\tau=\varnothing$, return \textbf{inconclusive}.
\item Otherwise return $\widehat j\in\arg\max_{j\in\mathcal C_\tau}\widehat v_j$ (break ties by index), with certificate $Q_{p,\widehat j}\le\tau$.
\end{enumerate}
\end{algorithm}

\paragraph{A practical check before collecting more data.}
Algorithm~\ref{alg:recovered} applies the repair without new outcomes. Utility scores may use the same sample, but need their own inputs and receive no utility guarantee here. If neither the required summaries nor raw logs remain, reconstruction is unavailable. Declaring success whenever any of several certificate methods succeeds requires additional multiplicity accounting. On RouterBench, the repair changes the deployed router at a fixed log size (Section~\ref{sec:experiments}).

\paragraph{TailCert for selected policy mixtures.}
To deploy a mixture, draw a policy from a learned distribution $\rho$ on each request. This distribution is called a posterior; $\bar F_\rho$ is the CDF of the resulting pooled outcomes. For an independently fixed policy family and prior, Appendix~\ref{app:cdf} constructs a PAC-Bayes Bernstein lower envelope $L_\rho(t)$, simultaneous over all admissible posteriors and a predeclared grid $\mathcal T$. TailCert returns
\[
 \widehat q^{\rm TailCert}_p(\rho)
 =\min\big\{t\in\mathcal T:L_\rho(t)\ge p\big\},
\]
or inconclusive when the set is empty. On the same $1-\delta$ event this bounds the pooled $p$-quantile, including at atoms (point masses), and permits selection under $L_\rho(\tau)\ge p$. This direct CDF construction is CDF-PB in the experiments; Exc-PB uses exceedances instead. The matched finite-class rates in Theorem~\ref{thm:boundary} are attained by empirical Bernstein, not asserted for TailCert.

Exc-PB and Exc-Bet are our adaptations of established PAC-Bayes and betting methods to weighted exceedances. Appendices~\ref{app:posterior_certificate}--\ref{app:pbpower} give Exc-PB's posterior-uniform guarantee and worst-case power bound, including the tuning-grid cost. Appendix~\ref{app:betting} specifies Exc-Bet; its reported power constants remain empirical.

\section{Experiments}\label{sec:experiments}
We test the argument in two stages. Paired simulations isolate whether information removal causes the predicted penalty and whether retaining counts removes it. Outcome-table studies then ask whether that distinction changes which useful policy can be deployed from a fixed log budget. EB, PB and Bet denote empirical-Bernstein, PAC-Bayes and betting bounds; CDF and Exc identify the statistic used.
\paragraph{Retaining counts restores the tail scaling.}
Does the observed penalty come from the confidence inequality or the information it receives? Figure~\ref{fig:information} compares rules on identical logs from Theorem~\ref{thm:boundary}. At fixed relative margin, the log size for $90\%$ power scales as $u^{-2.00}$ for the exact CDF test, versus $u^{-1.00}$ to $u^{-1.03}$ for every tested rule retaining counts or exceedances. The exact CDF/count ratio grows on the predicted order $p/u$, including when the candidate count varies at fixed overlap (Table~\ref{tab:infostudy}). Thus even the exact CDF test retains the penalty, while adding counts changes the scaling.

We compare three information classes: joint CDF indicators only, exceedance indicators, and CDF indicators with action counts. Each rule returns the first certified candidate in a fixed order, so CDF-only selection uses only the stipulated reduction. The study uses $\gamma=u/2$, $\delta=0.05$ and $2000$ repetitions per log size; Appendix~\ref{app:infostudy} specifies the rules. With one candidate, the penalty vanishes on-policy and reappears under poor overlap. Generic certificates have at most $1.1\%$ observed false certification under the tested nulls; Table~\ref{tab:infovalidity} gives the exact tests' nominal validity and Monte Carlo results.

\begin{figure*}[t]
\centering
\includegraphics[width=0.98\textwidth]{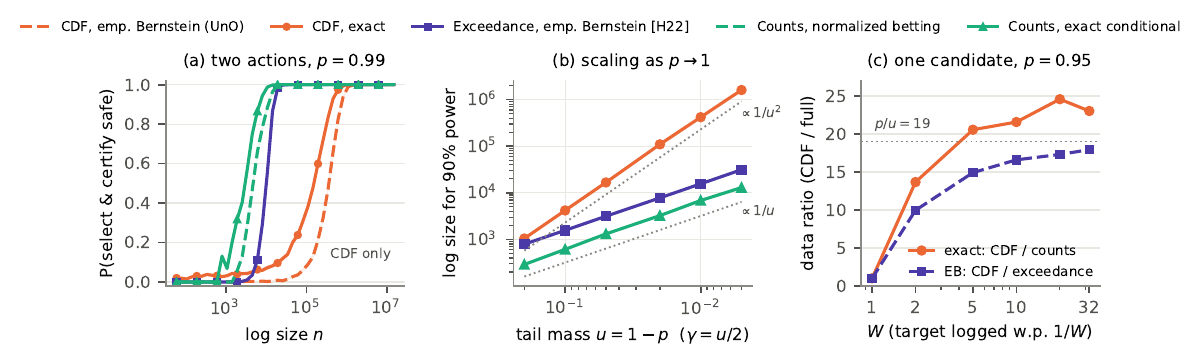}
\caption{\textbf{Retaining counts restores inverse-tail scaling on the family of Theorem~\ref{thm:boundary}} ($\gamma=u/2$, $\delta=0.05$, $2000$ repetitions per point). Orange: rules seeing only the joint CDF indicators; purple: exceedance indicators; green: action counts in addition. (a) Probability of selecting and certifying the safe candidate with two uniformly logged actions at $p=0.99$; shading gives pointwise $95\%$ Wilson intervals over simulation repetitions. (b) Log size for $90\%$ power against the tail mass $u$: Plotted CDF slopes are $-1.95$ (CDF-EB) and $-2.00$ (CDF-exact), versus $-1.00$ to $-1.03$ for the other plotted rules. (c) One candidate logged with probability $1/W$ at $p=0.95$: the exact CDF/count ratio is $1$ on-policy and $21$--$25$ for $W\ge5$, comparable in order to $p/u=19$.}
\label{fig:information}
\end{figure*}

\paragraph{Search costs persist when policies share actions.}
The effect does not require a separate action for every candidate. For $125$ affine routers sharing five uniformly logged actions (Corollary~\ref{cor:affine}), increasing the candidate list from $2$ to $125$ raises Exc-EB's $90\%$-power log size from $15$k to $27$k at $p=0.90$. Across the tested lists and both tail levels, dividing log size by $V\log(K/\delta)/\gamma^2$, with $V=5(u+\gamma)$, gives comparable values; the CDF/count penalty also persists. A repeated-candidate control illustrates the scope limit: a nominal union bound pays for duplicate routers even though duplicates add no selection problem. Appendix~\ref{app:infostudy} reports the full comparisons.

\paragraph{Selection must be included in the certificate.}
A separate rare-action simulation checks the predicted normalization, search complexity and data reuse. Exceedance power curves nearly coincide after scaling log size by $n^\star=W(1-p)\log(K/\delta)/\gamma^2$. Exc-EB reaches $90\%$ power at $(4.7$--$5.4)n^\star$; the other tested exceedance certificates differ mainly in constants. The naive rule omitting search and tuning penalties falsely certifies about $10\%$ at nominal $5\%$, while splitting needs $1.4$--$1.9\times$ the log size of full-sample Exc-PB. These measurements concern the implemented certificates. Appendix~\ref{app:empirical_details} gives settings, curves and the remaining method comparisons.

\paragraph{From recovered bounds to deployment decisions.}
The outcome-table studies ask whether recovery changes which useful policy can be deployed from a fixed log budget. Reconstructed and directly computed Exc-EB bounds agree to $10^{-14}$ and yield identical decisions on both tables (Appendix~\ref{app:infostudy}). We therefore compare that recoverable certificate with CDF certificates and other exceedance methods, all on the same logs at nominal $\delta=0.05$.

Candidates and budgets are fixed on development data. Selection maximizes importance-weighted utility among certified candidates; if none certifies, it returns an uncertified reference policy and records abstention. Utility selection uses full logs, so the CDF-based variants here are valid simultaneous-certificate comparisons, not selectors restricted to Theorem~\ref{thm:boundary}(b)'s observation model. Appendix~\ref{app:empirical_details} gives the full protocols and comparator definitions.

\paragraph{Recovery changes the deployed router.}
With $3{,}640$ logs, recovered Exc-EB permits mean deployed quality $0.607$, whereas the CDF certificates almost always return the reference router, whose quality is $0.305$ (Table~\ref{tab:fixedbudget}). The same pattern holds at $p=0.99$. These are decisions from the same observed outcomes: exact reconstruction changes what the pipeline can certify. No valid procedure falsely certified in any cell of the main study.

RouterBench contains $36{,}481$ prompts and five models; separate development and evaluation splits define routers trading quality against dollar cost. Table~\ref{tab:router} explains the decision gap: at $p=0.95$ under uniform logging, Exc-EB needs $1.9$k logs for $90\%$ certification probability, versus $26.7$k for CDF-EB. Norm-Bet matches Exc-Bet at the displayed precision, so this gap does not extend to the normalization-aware comparator.

\begin{table}[t]
\centering\small
\setlength{\tabcolsep}{3pt}
\caption{\textbf{Data needed for a tail constraint.} Estimated log sizes in thousands (lower is better), interpolated at $90\%$ empirical certification on RouterBench ($K=32$, $\delta=0.05$). Uniform exploration probability $\epsilon$ gives $W=5/\epsilon$. ``--'': not reached within $300$k. PB: PAC-Bayes; Bet: betting; EB: empirical Bernstein. CDF columns are unnormalized. Norm-Bet matches Exc-Bet at the displayed precision. $30$ log sizes (about $25\%$ steps), $300$ repetitions. Bold compares the recoverable EB and CDF-EB routes under uniform logging.}
\label{tab:router}
\begin{tabular}{llrrrrr}
\toprule
& & \multicolumn{3}{c}{\cellcolor{exchead}\textbf{Exceedance}} & \multicolumn{2}{c}{\cellcolor{cdfhead}\textbf{CDF}}\\
\cmidrule(lr){3-5}\cmidrule(lr){6-7}
\rowcolor{tablehead}
$p$ & $\epsilon$ & \textbf{PB} & \textbf{Bet} & \textbf{EB} & \textbf{EB} & \textbf{PB}\\
\midrule
0.90 & 1 & 1.2 & 0.8 & \keycell{1.0} & \keycell{7.3} & 11.5\\
0.90 & 0.5 & 2.3 & 1.5 & 1.8 & 4.4 & 22.5\\
0.90 & 0.2 & 5.5 & 3.6 & 4.4 & 6.5 & 58.8\\\addlinespace[3pt]
0.95 & 1 & 2.3 & 1.5 & \keycell{1.9} & \keycell{26.7} & 43.4\\
0.95 & 0.5 & 4.4 & 2.9 & 3.6 & 12.7 & 83.5\\
0.95 & 0.2 & 10.7 & 6.9 & 8.6 & 13.7 & 220.6\\\addlinespace[3pt]
0.99 & 1 & 10.7 & 6.9 & \keycell{8.6} & \keycell{--} & --\\
0.99 & 0.5 & 25.9 & 13.4 & 16.7 & 184.9 & --\\
0.99 & 0.2 & 62.7 & 32.3 & 50.3 & 119.7 & --\\\bottomrule
\end{tabular}
\end{table}
\begin{table}[t]
\centering
\caption{\textbf{Reprocessing changes the deployment decision.} RouterBench, uniform logging, $K=32$, $\delta=0.05$: mean deployed quality $\pm$ Monte Carlo standard error over $300$ paired logs, including fallback, and certification frequency. Standard errors describe each method's mean, not a paired difference or uncertainty in the benchmark population. Bold and shading mark exact recovery of Exc-EB, not statistical superiority. All listed methods had zero observed unsafe certifications.}
\label{tab:fixedbudget}
\small
\setlength{\tabcolsep}{1.8pt}
\begin{tabular}{lcc}
\toprule
\rowcolor{tablehead}
\textbf{Method} & \textbf{Deployed quality} & \textbf{Cert. rate}\\
\midrule
\rowcolor{tablehead}\multicolumn{3}{l}{$p=0.95$, $n=3{,}640$}\\
\rowcolor{recoveryrow}\textbf{Exc-EB / recovered} & \keycell{$0.6069\pm0.0004$} & \keycell{$100.0\%$}\\
Norm-Bet & $0.6084\pm0.0004$ & $100.0\%$\\
CDF-EB & $0.3052\pm0.0007$ & $0.3\%$\\
CDF-PB & $0.3045\pm0.0000$ & $0.0\%$\\
\addlinespace
\rowcolor{tablehead}\multicolumn{3}{l}{$p=0.99$, $n=13{,}676$}\\
\rowcolor{recoveryrow}\textbf{Exc-EB / recovered} & \keycell{$0.6089\pm0.0001$} & \keycell{$100.0\%$}\\
Norm-Bet & $0.6093\pm0.0001$ & $100.0\%$\\
CDF-EB & $0.3055\pm0.0010$ & $0.3\%$\\
CDF-PB & $0.3045\pm0.0000$ & $0.0\%$\\
\addlinespace\bottomrule
\end{tabular}
\end{table}

\paragraph{The same pattern appears under token budgets.}
The reasoning study also finds much lower certification costs for exceedance-based methods. At $p=0.90$ under uniform logging, Exc-Bet needs 4.3k requests and CDF-EB 46.2k. At $p=0.95$, neither tested CDF method reaches $90\%$ certification within 98k for any tested logger, while Exc-Bet needs 7.1k--34.4k. Norm-Bet is within $6\%$ of Exc-Bet; Table~\ref{tab:gpu} gives the remaining comparisons. Seven open-weight configurations supply 98,000 generations on 14,000 problems. Development fixes budgets and routers; replay logs resample the 9,800-problem evaluation pool and certify that finite population.

A separate held-out benchmark evaluation uses a log containing each evaluation problem once with randomized logged actions, a stratified certificate, and fresh generations on 2,000 held-out problems. Under the stated model that each dataset supplies i.i.d.\ requests and their mixture has GSM8K share $q=0.413$, the selected router's certified exceedance bound is $0.0968$ at the $0.10$ budget; the reported held-out exceedance is $102/2000=0.051$. This interpretation depends on that population model and is separate from finite-pool replay. Appendix~\ref{app:empirical_details} gives the design and all held-out measurements.

\paragraph{Explaining log sizes and designing logging.}
Population-overlap calculations explain the reported median thresholds within a factor of two across the tested configurations (Appendix~\ref{app:predict}). These calculations use evaluation-population moments and are retrospective, not validated prospective sample-size forecasts. For fixed stochastic targets mixing different tail risks, tail-aware logging reduces the reported Exc-Bet log size by $2.3$--$6.0$ times on RouterBench at $10$--$20\%$ relative margins and by $1.1$--$1.3$ times in the reasoning study (Appendix~\ref{app:tailaware}). Gains depend on tail heterogeneity and on whether the variance term dominates the range term, as Proposition~\ref{prop:logging} suggests.

\section{Related work}\label{sec:related}
\paragraph{Off-policy distributions and risks.}
UnO constructs CDF bands from pointwise off-policy intervals \citep{chandak2021universal,thomas2015high}; OPRA develops importance-weighted and doubly robust distribution estimates \citep{huang2021off,huang2022mdp}. Most directly, \citet{huang2022offarxiv} introduce the complementary-CDF estimator, explain its tail-dependent variance, give empirical-Bernstein bounds, and combine CDF and complementary-CDF estimates using their empirical variances. We use these established estimators. Our principal result compares the data required by full logs and by the joint CDF-indicator reduction for the same selection-and-certification problem. It establishes an information restriction, rather than a limitation of every method described as CDF-based. Uniform policy-class risk bounds \citep{wan2026riskaware}, off-policy prediction intervals \citep{taufiq2022conformal,zhang2023conformal,wan2025pacopp}, and selected-policy score-quantile certificates \citep{stepcops2026} address related targets; score-distribution guarantees require additional justification to describe deployment outcomes.

\citet{luo2026boosting} use quantile gradients to improve sample efficiency in CVaR policy optimization in MDPs; we study finite-sample certification after selection from contextual-bandit logs.

Anytime-valid off-policy inference already provides CDF bands under adaptive logging and without a known global weight bound \citep{waudby2024anytime}. Our fixed-sample comparison asks how much information remains after a specified reduction of the log, and which additional summaries recover an established tail certificate.

\paragraph{Confidence bounds and policy selection.}
At a fixed threshold, our target is an off-policy Bernoulli mean. Existing self-normalized inference already supports policy selection \citep{kuzborskij2021confident}. Off-policy confidence sequences jointly exploit the known weight mean and the weighted outcome mean \citep{karampatziakis2021offpolicy}. Exc-PB uses nonnegative-loss PAC-Bayes concentration \citep{maurer2003bound,thiemann2017strongly,gabbianelli2024importance}; Exc-Bet uses predictable plug-in betting \citep{waudby2024estimating}. The contribution is the comparison of attainable and necessary certification costs, not selection validity alone.

\paragraph{Converse arguments and their scope.}
Our proofs use standard testing and Fano arguments, as in off-policy estimation \citep{li2015toward} and arm identification \citep{kaufmann2016complexity,kano2019good}. Conditional on a nonzero observation, the joint-CDF construction is a categorical mode-identification problem, connecting its KL calculation to \citet{jain2022mode}. The new comparison concerns what is lost when the same off-policy logs are reduced to these categories. The shared-action lower bound extends the local overlap argument under a quantitative separation condition; it does not characterize arbitrary redundant policy classes. Propensity overfitting \citep{swaminathan2015self} also motivates the caution about selected-mixture overlap in Appendix~\ref{app:mixture}.

\section{Discussion and limitations}\label{sec:discussion}
\paragraph{Practical consequences.}
Before collecting more outcomes after a failed certificate, check which summaries the procedure used. If CDF moments and the two weight totals, or the raw logs, remain available, Section~\ref{sec:repair} reconstructs Exc-EB without new measurements. Its bound may certify a policy that the CDF certificate could not. Normalized, control-variate, and model-assisted procedures can also exploit information excluded by Theorem~\ref{thm:boundary}(b), as the Norm-Bet comparisons illustrate.

\paragraph{What this enables and leaves open.}
A monitoring pipeline must decide which summaries to retain before discarding request-level logs. For fixed candidates and stored thresholds, two normalization totals per candidate preserve access to Exc-EB from CDF moments. This supports the cost and token-budget decisions studied here; latency budgets fit the same formulation if the assumptions hold. The broader question is which summaries preserve other downstream decisions. We establish sufficiency in the binary family and recover one existing certificate for general weights. Preserving arbitrary future candidates, utility criteria, or sequential decisions remains open.

\paragraph{Necessity and assumptions.}
The matched finite-family comparison assumes $W\ge K$ and a unique safe candidate. The local lower bound permits shared actions under the stated tail cross-moment condition and comparable overlaps. Neither establishes an instance-wise $\log K$ cost for arbitrary router families: duplicates and substantial shared information can reduce search complexity. Guarantees assume known propensities, bounded weights, independently fixed candidates or prior, and a predeclared threshold grid. They concern one decision per request, not sequential decisions with trajectory weights.

\paragraph{Evidence and logging design.}
Paired simulations test the information comparison directly; the outcome-table studies show that exact recovery changes which router can be deployed. These experiments establish decision consequences of the specified information choices, not how frequently CDF-only summaries are used in practice. Replay certifies finite evaluation populations; held-out generations require the separate population assumptions in Section~\ref{sec:experiments}. Calculations using evaluation-population moments explain log sizes retrospectively; prospective planning needs an independent pilot and uncertainty accounting. Tail-aware logging minimizes the tail second moment for a fixed target, but sample-size gains also depend on the range term and certificate. Designing one logger for an unknown selected policy is a further problem.

\section{Conclusion}
On an explicit finite-candidate family, reducing off-policy logs to joint CDF indicators erases the action identities of rare failures and incurs an order-$p/(1-p)$ certification penalty. Retaining action counts removes this penalty in the binary model; for general contextual weights, two normalization totals recover the established empirical-Bernstein exceedance certificate. Experiments show that recovery can change deployment decisions.

\bibliography{refs_aistats}
\clearpage
\onecolumn
\appendix
\makeatletter
\setlength{\@fptop}{0pt}
\setlength{\@fpsep}{18pt plus 2pt minus 2pt}
\setlength{\@fpbot}{0pt plus 1fil}
\makeatother
\section*{Supplement}
\begin{itemize}[leftmargin=1.15em,itemsep=0.25em,topsep=0.3em]
\item \textbf{Appendix~\ref{app:boundary}: information separation.}
Proves both parts of Theorem~\ref{thm:boundary}, using one concentration tool and two information bounds. Start here to check the claimed sample-size gap.
\item \textbf{Appendix~\ref{app:recovery}: recovery and intuition.}
Proves Proposition~\ref{prop:recovery}, explains the information erased in Figure~\ref{fig:two_action_recovery}, and gives the on-policy limit.
\item \textbf{Appendix~\ref{app:tailoverlap}: shared actions.}
Supplies the extension summarized in Section~\ref{sec:boundary}, including the five-action router construction and its scope conditions.
\item \textbf{Appendix~\ref{app:certificates}: certificate formulas.}
Defines the posterior, CDF, betting, and normalization-aware procedures supporting Sections~\ref{sec:repair} and~\ref{sec:experiments}, with their validity and power qualifications.
\item \textbf{Appendix~\ref{app:design_scope}: design and limits.}
Derives the fixed-target logging consequence and explains mixture-overlap limitations and the scope of the information comparison discussed in Sections~\ref{sec:related}--\ref{sec:discussion}.
\item \textbf{Appendix~\ref{app:experiments_support}: empirical details.}
Documents the protocols, reference procedures, tables, and supplementary plots behind Section~\ref{sec:experiments}; it distinguishes fixed-pool replay from population deployment.
\end{itemize}

\paragraph{Notation conventions.}
We inherit the notation of Section~\ref{sec:setup}. All logarithms are natural. $\KL(P\|Q)$ denotes relative entropy, and $\kl(a\|b)=a\log(a/b)+(1-a)\log((1-a)/(1-b))$ is its Bernoulli form, with the usual continuous boundary conventions. $P^{\otimes n}$ is the law of $n$ independent draws from $P$, and $I(U;V)$ denotes mutual information. We write $\rho\ll\pi$ for absolute continuity and $\delta_\theta$ for a point mass at $\theta$. The comparisons $a\lesssim b$ and $a\asymp b$ mean, respectively, an upper bound and two-sided bounds up to positive constants in the stated regime. Expectations without a subscript use the current model's logging law.

\section{Proof of the information separation}\label{app:boundary}
This section proves Theorem~\ref{thm:boundary}. The same empirical-Bernstein tool supplies both achievable bounds; testing arguments then show how the information restriction changes the necessary log size.

\subsection{The concentration tool}\label{app:eb}
The following established confidence bound and its power calculation are used in the main theorem, the reconstruction recipe in Section~\ref{sec:repair}, and the shared-action extension in Appendix~\ref{app:tailoverlap}.

\begin{lemma}[Empirical-Bernstein certification and power]\label{lem:eb}
Let $n\ge2$ and $X_1,\ldots,X_n$ be IID in $[0,W]$, $W\ge1$, with mean $\mu$ and $\mathrm{Var}(X_i)\le V$. Write $\bar X=n^{-1}\sum_{i=1}^nX_i$ and $S_n^2=(n-1)^{-1}\sum_{i=1}^n(X_i-\bar X)^2$ for the sample mean and variance. For $\alpha\in(0,1)$, put $\ell=\log(2/\alpha)$ and
\[
 r_n=\sqrt{2S_n^2\ell/n}+\frac{7W\ell}{3(n-1)}.
\]
Each bound $\bar X+r_n$ and $\bar X-r_n$ has one-sided error probability at most $\alpha$. Suppose the mean is at least $h\in(0,1]$ below an upper certification threshold, or at least $h$ above a lower threshold. If $\eta\in(0,1)$ and
\[
 L=\log\frac{2}{\min(\alpha,\eta/2)},\qquad
 n\ge\max\!\left\{\frac{32VL}{h^2},\frac{11WL}{h}\right\},
\]
then the appropriate bound crosses its threshold with probability at least $1-\eta$.
\end{lemma}
\begin{proof}
One-sided validity is the empirical-Bernstein inequality of \citet[Theorem~4]{maurer2009empirical}, applied also to $W-X_i$ for the opposite direction. With probability at least $1-\eta/2$, Bernstein's inequality gives the required one-sided empirical-mean deviation at most $\sqrt{2VL/n}+WL/(3n)$. With the same failure budget, concentration of the sample standard deviation \citep[Theorem~10]{maurer2009empirical} gives
\[
 \sqrt{S_n^2}\le\sqrt V+W\sqrt{2L/(n-1)}.
\]
On their intersection, the total distance from the mean to the relevant confidence bound is at most
\begin{align*}
 &2\sqrt{\frac{2VL}{n}}+
 \frac{WL}{3n}+\frac{2WL}{\sqrt{n(n-1)}}+
 \frac{7WL}{3(n-1)}\\
 &\hspace{1em}\le2\sqrt{\frac{2VL}{n}}+\frac{5.2WL}{n}.
\end{align*}
The sample-size assumptions force $n\ge11$, using $W\ge1$, $h\le1$ and $L\ge\log4$, so the final numerical inequality holds. The first term is at most $h/2$, and the second is at most $5.2h/11<h/2$. A union bound proves the power statement.
\end{proof}

\subsection{Full logs: proof of Theorem~\ref{thm:boundary}(a)}\label{app:fullproof}
The upper bound uses the concentration lemma above. The lower bound separates the confidence cost by a two-model test and the search cost by Fano's inequality.
\paragraph{Sufficiency.}
For each candidate, apply the empirical-Bernstein upper bound in Lemma~\ref{lem:eb} at level $\alpha=\delta_*/K$. A union bound gives simultaneous validity with failure probability at most $\delta_*\le\delta$. Apply the lemma to the uniquely safe candidate $s$ with power failure $\eta=\beta/2$. Its logarithmic term is
\[
 \log\frac{2}{\min(\delta_*/K,\beta/4)}
 =\log(2K/\delta_*)=L_*,
\]
because $K\ge2$ and $\delta_*\le\beta/2$. For its summands $w_s(X_i,A_i)Z_i$, the second moment obeys $V_s\le WG_s\le Wu$. Condition~\eqref{eq:upper} implies both conditions in Lemma~\ref{lem:eb}: the range condition follows from $\gamma\le u/2$. Thus the safe candidate is certified except on an event of probability $\beta/2$. Outside this event and the simultaneous-validity failure, it is the only possible output. The total failure probability is at most $\beta/2+\delta_*\le\beta$.

\paragraph{Necessity.}
Under $P_j$ and $P_0$, only action $j$ has a different conditional outcome law. Consequently,
\[
 \KL(P_j^{\otimes n}\,\|\,P_0^{\otimes n})
 =\frac nW\kl \big(u-\gamma \,\|\, u+\gamma \big)
 \le\frac{16n\gamma^2}{Wu}.
\]
Indeed, $\kl(a\|b)\le(a-b)^2/[b(1-b)]$, $b=u+\gamma\ge u$, and $1-b=p-\gamma\ge1/4$. Under $P_0$ every output is unsafe, so $P_0(\Psi=j)\le\delta$. Under $P_j$, $P_j(\Psi=j)\ge1-\beta>\delta$. Data processing for this event and monotonicity of binary KL give
\[
 \kl (1-\beta \,\|\, \delta)
 \le\KL \big(P_j^{\otimes n}\,\|\,P_0^{\otimes n} \big).
\]
For the search cost, let $J$ be uniform on $\{1,\ldots,K\}$, draw $\calD_n\mid J=j\sim P_j^{\otimes n}$, and map abstention to an arbitrary label. The resulting estimate has error probability at most $\beta$. Fano's inequality yields
\[
 (1-\beta)\log K-\log2
 \le I(J;\calD_n)
 \le\frac1K\sum_j\KL \big(P_j^{\otimes n}\,\|\,P_0^{\otimes n} \big).
\]
Independent randomization of the decision rule leaves these information bounds unchanged. Combining the two inequalities proves~\eqref{eq:lower}.

\paragraph{Order statement.}
For fixed $0<\beta\le1/4$ and $0<\delta\le1/4$, $L_*$ is within constant factors of $\log K+\log(1/\delta)$, with constants allowed to depend on the fixed power. Moreover,
\[
 \kl \big(1-\beta \,\|\, \delta \big)
 \ge(1-\beta)\log(1/\delta)-\log2.
\]
Together with the Fano term, this lower bounds a constant multiple of $\log K+\log(1/\delta)$; for bounded $K$, the confidence term absorbs the bounded search term. These observations give the stated minimax order. No uniform order claim is made as $\delta$ approaches $1-\beta$.

\subsection{Joint CDF data: proof of Theorem~\ref{thm:boundary}(b)}\label{app:penalty}
This proof changes only the observation available to the decision rule. Collapsing all unlabelled failures into one category reduces the information per observation.
Write $Q_j$ for the one-observation law of the entire CDF vector $T$. It is categorical on $\{0,We_1,\ldots,We_K\}$, where $e_j$ is the $j$th standard basis vector in $\mathbb R^K$. Its masses are
\[
 Q_j(We_j)=\frac{p+\gamma}{W},\qquad
 Q_j(We_k)=\frac{p-\gamma}{W}\quad(k\ne j),
\]
\[
 Q_j(0)=1-\frac{Kp-(K-2)\gamma}{W}\ge1-p>0.
\]
The zero mass is the same for every alternative. For $j\ne k$, only two nonzero categories change, and
\begin{align*}
 \KL(Q_j\|Q_k)
 &=\frac{2\gamma}{W}\log\frac{p+\gamma}{p-\gamma}\\
 &\le\frac{4\gamma^2}{W(p-\gamma)}
 \le\frac{8\gamma^2}{Wp}.
\end{align*}
Here $\log(1+x)\le x$ and $\gamma\le u/2\le p/2$.

\paragraph{Confidence and search.}
Fix $k\ne j$. Candidate $j$ is unsafe under $P_k$, so validity implies $Q_k^{\otimes n}(\Psi=j)\le\delta$, while power implies $Q_j^{\otimes n}(\Psi=j)\ge1-\beta$. Data processing gives
\[
 \kl(1-\beta\,\|\,\delta)
 \le n\KL(Q_j \,\|\,Q_k)
 \le\frac{8n\gamma^2}{Wp}.
\]
Write $T_{1:n}=(T_1,\ldots,T_n)$. Taking $J$ uniform on $\{1,\ldots,K\}$ with $T_{1:n}\mid J=j\sim Q_j^{\otimes n}$ and again assigning an arbitrary label on abstention, Fano's inequality and the reference law $Q_1^{\otimes n}$ give
\begin{align*}
 (1-\beta)\log K-\log2&\le I(J;T_{1:n})\\
 &\le\frac1K\sum_j n\KL(Q_j \,\|\, Q_1)\\
 &\le\frac{8n\gamma^2}{Wp}.
\end{align*}
This proves~\eqref{eq:cdflower}, including for randomized rules. The argument uses validity only on the alternatives; inclusion of $P_0$ makes parts (a) and (b) of Theorem~\ref{thm:boundary} use exactly the same model class.

\paragraph{Matching upper bound.}
Let $X_{ij}=w_j(X_i,A_i)\1\{Y_i\le\tau\}$. Use empirical-Bernstein lower bounds at level $\delta_*/K$ and output any candidate whose bound is at least $p$. For the safe candidate, $\E X_{ij}=p+\gamma$ and $\E X_{ij}^2=W(p+\gamma)\le3Wp/2$. The lower-bound version of Lemma~\ref{lem:eb}, with failure $\beta/2$, uses $L_*$. Thus $n\ge48WpL_*/\gamma^2$ suffices for its certification, and also dominates $11WL_*/\gamma$. As in the full-data proof, the union of safe-candidate certification failure and simultaneous-validity failure has probability at most $\beta$. This proves the claimed validity and power on every $P_j$.

\subsection{Boundary case: one candidate}\label{app:singlecandidate}
The following separate calculation isolates the confidence cost when there is no search over policies. It complements the one-candidate comparison in Figure~\ref{fig:information}(c), rather than deriving the $K$-candidate search term.
\begin{proposition}[Single-candidate CDF reduction]\label{prop:cdfsingle}
For one deterministic target action logged with probability $1/W$, $W\ge2$, let $0<p<1$, $0<\gamma\le1-p$, and $0<\delta<1-\beta<1$. Write $a_\star$ for the target action, $w_i=W\1\{A_i=a_\star\}$, and $F(t)$ for the outcome CDF under the policy that always chooses $a_\star$. A possibly randomized rule depending only on $(w_i\1\{Y_i\le\tau\})_{i=1}^n$, valid at level $\delta$ for every $F(\tau)<p$, and certifying $F(\tau)=p+\gamma$ with probability at least $1-\beta$, requires
\[
 n\ge\frac{Wp}{2\gamma^2}\kl(1-\beta \,\|\, \delta).
\]
\end{proposition}
\begin{proof}
Put $F=F(\tau)$. The retained observations have law $W\cdot\mathrm{Bernoulli}(F/W)$. For fixed $n$, the probability of certification is continuous in $F$, also for a randomized rule, because it is a finite sum of polynomial sampling probabilities with fixed decision probabilities. Validity for $F<p$ therefore implies validity at the boundary $F=p$. Data processing gives
\begin{align*}
 \kl(1-\beta\|\delta)
 &\le n\kl\!\left(\frac{p+\gamma}{W}\middle\|\frac pW\right)\\
 &\le\frac{n\gamma^2}{Wp(1-p/W)}
 \le\frac{2n\gamma^2}{Wp}.
\end{align*}
Rearrangement proves the result.
\end{proof}

\section{Recovery and the information carried by failures}\label{app:recovery}
The first subsection proves the count-recovery proposition behind Figure~\ref{fig:two_action_recovery}. The remaining calculations explain the exact information ratio and its on-policy limit; they are not additional assumptions for using the recovery recipe.

\subsection{Proof of Proposition~\ref{prop:recovery}}\label{app:countproof}
For arbitrary target-action tail probabilities $g=(g_1,\ldots,g_K)$, the full-data likelihood has the form
\[
 h(\calD_n)\prod_{j=1}^K
 g_j^{N_j-S_j}(1-g_j)^{S_j},
\]
where $h$ contains logging probabilities and the fixed dummy outcome law and does not depend on $g$. The factorization criterion proves sufficiency of $(N,S)$. The vector $N$ has a known multinomial law independent of $g$; conditional on $N$, the counts $N_j-S_j$ are independent $\mathrm{Bin}(N_j,g_j)$. Exact one-sided binomial confidence bounds with error $\delta_*/K$ therefore give valid simultaneous inference; a candidate with $N_j=0$ receives upper bound one. To obtain the proved upper bound~\eqref{eq:upper}, put $E_j=N_j-S_j$ and reconstruct
\[
 \widehat G_j=\frac{WE_j}{n},\qquad
 \widehat\sigma_j^2=\frac{W^2(E_j-E_j^2/n)}{n-1}.
\]
The variance formula uses $n$ observations taking values in $\{0,W\}$. Applying the same Exc-EB bound therefore recovers~\eqref{eq:upper}. This sufficiency argument does not assert optimal constants for every normalized interval implementation.

\subsection{Exact comparison of the induced experiments}\label{app:information}
For distinct alternatives $j,k$, let $D_F$ and $D_E$ be their one-observation KL divergences after retaining, respectively, the joint CDF and joint exceedance indicators. Then
\begin{equation}\label{eq:reducedkl}
\begin{aligned}
 D_F&=\frac{2\gamma}{W}\log\frac{p+\gamma}{p-\gamma},\\
 D_E&=\frac{2\gamma}{W}\log\frac{u+\gamma}{u-\gamma},\\
 \KL(P_j\|P_k)&=D_F+D_E.
\end{aligned}
\end{equation}

The calculation of $D_F$ is in Appendix~\ref{app:penalty}. The joint exceedance vector has mass $(u-\gamma)/W$ at $We_j$, mass $(u+\gamma)/W$ at $We_k$ for $k\ne j$, and a common zero mass. The same two-category calculation gives $D_E$ in~\eqref{eq:reducedkl}. For full logs, the contribution of the two changed actions is
\[
 \KL(P_j\,\|\,P_k)
 =\frac1W\{\kl(u-\gamma\,\|\,u+\gamma)
          +\kl(u+\gamma\,\|\,u-\gamma)\}.
\]
Separating success and exceedance terms gives $D_F+D_E$. This equality is a property of the construction and does not assert a general additive decomposition for statistical reductions.

\paragraph{Two tail regimes.}
For fixed $p$ and $\gamma\to0$, expansion of the logarithms gives $D_F\sim4\gamma^2/(Wp)$ and $D_E\sim4\gamma^2/(Wu)$, hence $D_E/D_F\to p/u$. When instead $\gamma=cu$ with fixed $c\in(0,1/2]$ and $u\to0$,
\[
 D_E=\frac{2cu}{W}\log\frac{1+c}{1-c},\qquad
 D_F\sim\frac{4c^2u^2}{Wp}.
\]
Thus the inverse-tail-probability information ratio also holds in the relative-margin regime, where a fixed-$p$ local expansion alone would not suffice.

\subsection{When the information gap disappears}\label{app:onpolicy}
For a target action logged with probability $q\in(0,1]$, compare tail probabilities $u-\gamma$ and $u$. The CDF-only, exceedance-only and full-data divergences are, respectively,
\[
 \kl \Big(q \big(p+\gamma \big) \,\|\, qp \Big),\quad
 \kl \Big(q \big(u-\gamma \big)\,\|\,qu \Big),\quad
 q\,\kl \big(u-\gamma \,\|\, u \big).
\]
For fixed $p,q$ and $\gamma\to0$, the ratio of exceedance to CDF divergence converges to $p(1-qp)/\{u(1-qu)\}$. This equals one at $q=1$ and tends to $p/u$ as $q\to0$. This is a separate one-candidate comparison, not a claim that the family with $K\ge2$ permits on-policy logging for every candidate simultaneously.

\section{Tail overlap and shared actions}\label{app:tailoverlap}\label{sec:tailoverlap}
This extension supports the scope discussion in Section~\ref{sec:boundary}. It gives a sufficient certificate bound for general fixed candidates and a matching necessity result on separated families sharing actions. The main CDF separation does not depend on this extension. Retain $u=1-p$, $\delta_*=\min\{\delta,\beta/2\}$, and $L_*=\log(2K/\delta_*)$ from the main text.

\subsection{Conditions and statement}\label{app:sharedstatement}

The preceding family has one private action per candidate. For general candidates, define $g(x,a)=\PP(Y>\tau\mid X=x,A=a)$ and the tail second moment
\[
 V_j=\E_{\pi_0}\big[w_j(X,A)^2Z \big]
     =\E \big[w_j(X,A)^2g(X,A) \big].
\]
Here and below, expectations of functions of $(X,A)$ use the logger's fixed context/action law. For discrete actions, equivalently,
\[
 V_j=\E_X\sum_{a:\pi_0(a\mid X)>0}
       \frac{\pi_j(a\mid X)^2}{\pi_0(a\mid X)}g(X,a).
\]
It is at most $WG_j$ and excludes weight variation on non-tail outcomes. A global range term can nevertheless remain in finite-sample bounds. The second moment $V_j$ measures how strongly rare violations are amplified by reweighting. The sufficient bound below has a variance term and a range term; reducing $V_j$ helps sample size only while the former is the bottleneck. For the converse, the normalized tail cross moment defined below measures how much a perturbation helping candidate $j$ also helps candidate $k$. Requiring this transfer to be less than one half lets $j$ become safe while every other candidate remains unsafe.

\begin{theorem}[Tail overlap and shared-action necessity]\label{thm:instance}
Under Assumption~\ref{ass:overlap}, fix $K\ge1$ candidates, $0<\delta<1-\beta<1$, and $L_K=\log\{2/\min(\delta_*/K,\beta/4)\}$, equal to $L_*$ when $K\ge2$.
\emph{(a) Sufficiency.} Simultaneous empirical-Bernstein exceedance upper bounds at levels $\delta_*/K$ certify a candidate $s$ of margin $\gamma_s=u-G_s>0$ with probability at least $1-\beta/2$ whenever
\[
 n\ge\max\!\left\{\frac{32V_sL_K}{\gamma_s^2},
                         \frac{11WL_K}{\gamma_s}\right\}.
\]
The rule that outputs any certified candidate is $\delta$-valid. If $s$ is uniquely safe, its output probability is at least $1-\beta$.
\emph{(b) Necessity.} Let $K\ge2$ and let $P_0$ be a binary-outcome base model, with $g_0(x,a)=\PP_{P_0}(Y>\tau\mid X=x,A=a)$. For any outcome model $P$, write $G_j(P)$ and $V_j(P)$ for candidate $j$'s risk and tail second moment under $P$. Assume $g_0\le3/4$ wherever some candidate weight is positive, $G_j(P_0)=u+\gamma$, and
\begin{align*}
 V_j^0&=\E \big[w_j^2g_0 \big]\ge3\gamma W,\\
 C_{kj}&:=\E \big[w_kw_jg_0 \big]\le\zeta V_j^0
\end{align*}
for every ordered pair $k\ne j$, where $\gamma>0$ and $0\le\zeta<1/2$. There exist alternatives $P_j$ with unique safe candidate $j$, margin $\gamma$, and $V_j(P_j)\in[V_j^0/3,V_j^0]$. Every rule $\delta$-valid on $\{P_0,P_1,\ldots,P_K\}$ and with output power $1-\beta$ under each $P_j$ requires
\[
 n\ge\frac1{16\gamma^2}\max\!\left\{\begin{aligned}
 &\max_jV_j^0\,\kl(1-\beta \,\|\,\delta),\\
 &V_H \big[(1-\beta)\log K-\log2 \big]
 \end{aligned}\right\},
\]
where $V_H=(K^{-1}\sum_j1/V_j^0)^{-1}$.
\emph{(c)} For comparable $V_j^0\asymp V$, fixed $0<\beta\le1/4$, and $0<\delta\le1/4$, the bounds match at order $V\log(K/\delta)/\gamma^2$ on this family.
\end{theorem}

Disjoint supports are the special case $C_{kj}=0$. The proof below sets $g^{(j)}=g_0(1-2\gamma w_j/V_j^0)$: the cross-moment condition ensures that improving candidate $j$ does not make another candidate safe. It is a sufficient separation condition, not a universal transition.

The construction in Appendix~\ref{app:affine} shows that necessity can cover $K\gg W$. Small cross moments alone do not verify the equal-baseline-risk assumptions in applications. The sufficient bound in part (a) applies to general lists; the necessity statement is limited to the constructed families and does not cover arbitrary lists containing duplicates.

\subsection{Proof of the shared-action theorem}\label{app:sharedproof}
\paragraph{Theorem~\ref{thm:instance}(a).}
Apply Lemma~\ref{lem:eb} to the summands $w_s(X_i,A_i)Z_i$, with variance at most $V_s$, candidate error $\alpha=\delta_*/K$, power failure $\eta=\beta/2$, and gap $h=\gamma_s$. The logarithmic term is $L_K$, equal to $L_*$ for $K\ge2$ as shown in Appendix~\ref{app:boundary}. The simultaneous-validity event has failure at most $\delta_*$. Hence the safe candidate is certified with failure at most $\beta/2$, and is output with failure at most $\beta$ when uniquely safe.

\paragraph{Theorem~\ref{thm:instance}(b).}
The assumptions give $3\gamma W\le V_j^0\le W(u+\gamma)$, hence $\gamma\le u/2$. Write $\kappa_j=2\gamma/V_j^0$. The condition $V_j^0\ge3\gamma W$ implies $0\le\kappa_jw_j\le2/3$. Keep the law of $(X,A)$ fixed and define
\[
 g^{(j)}(x,a)=g_0(x,a)\{1-\kappa_jw_j(x,a)\}.
\]
It lies in $[g_0/3,g_0]$, so it defines a valid binary-outcome model. The target means become
\begin{align*}
 G_j(P_j)&=u+\gamma-\kappa_jV_j^0=u-\gamma,\\
 G_k(P_j)&=u+\gamma-\kappa_jC_{kj}
            \ge u+(1-2\zeta)\gamma>u\quad(k\ne j).
\end{align*}
Thus $j$ is uniquely safe. Also $V_j(P_j)=\E[w_j^2g^{(j)}]\in[V_j^0/3,V_j^0]$. The laws agree wherever $w_jg_0=0$. Elsewhere $0<g_0\le3/4$, so Bernoulli KL and $1-g_0\ge1/4$ give
\[
 \KL(P_j \,\|\,P_0)
 \le\E\frac{(\kappa_jw_jg_0)^2}{g_0(1-g_0)}
 \le4\kappa_j^2V_j^0
 =\frac{16\gamma^2}{V_j^0}.
\]
where the integrand is defined as zero on the agreement set. Under $P_0$ all candidates are unsafe. Data processing on $\{\Psi=j\}$ gives
\[
 \kl(1-\beta \,\|\, \delta)\le\frac{16n\gamma^2}{V_j^0}
 \quad\text{for every }j.
\]
For uniform $J\in\{1,\ldots,K\}$ with $\calD_n\mid J=j\sim P_j^{\otimes n}$, Fano and the reference law $P_0^{\otimes n}$ give
\[
 (1-\beta)\log K-\log2
 \le\frac1K\sum_j\KL(P_j^{\otimes n}\,\|\, P_0^{\otimes n})
 \le\frac{16n\gamma^2}{V_H}.
\]
Combining these bounds proves part (b).

\paragraph{Theorem~\ref{thm:instance}(c).}
Under $P_j$, part (a) applies with $V_j(P_j)\le V_j^0$ and margin $\gamma$. Since $V_j^0\ge3\gamma W$, the variance sample-size term dominates the range term. Comparable $V_j^0$ imply comparable $V_H$, and the order argument from Appendix~\ref{app:boundary} applies. This proves matching order on the constructed family, not on every instance sharing its second moments.

\subsection{A concrete family over five actions}\label{app:affine}
This example shows that the shared-action result allows many more candidates than logged actions. It supplies the shared-action construction used in Appendix~\ref{app:infostudy}.

\begin{corollary}[Many routers sharing five actions]\label{cor:affine}
Identify the five actions with $\mathbb F_5=\{0,1,2,3,4\}$, with arithmetic modulo five, log them uniformly, and take contexts uniform on $\mathbb F_5^2$. For $(a,b,c)\in\mathbb F_5^3$ and $x=(x_1,x_2)$, the 125 deterministic routers $f_{a,b,c}(x)=ax_1+bx_2+c\pmod5$ have pairwise agreement at most $1/5$. With constant $g_0=u+\gamma\le3/4$ and $0<\gamma\le u/2$, they satisfy Theorem~\ref{thm:instance}(b): $W=5$, $V_j^0=5(u+\gamma)$, and $C_{kj}/V_j^0\le1/5$. Every fixed subset of at least two routers retains the guarantee.
\end{corollary}

\paragraph{Proof.}
Enumerate these routers as $f_1,\ldots,f_{125}$. Two distinct affine functions with equal slopes have no agreements. If their slopes differ, their equality is one nontrivial linear equation on $\mathbb F_5^2$, which has exactly five solutions. Uniform logging gives $w_j=5\1\{A=f_j(X)\}$, and constant $g_0$ yields
\[
 V_j^0=5(u+\gamma),\qquad
 C_{kj}=5(u+\gamma)\PP_X\{f_k(X)=f_j(X)\}.
\]
Thus $C_{kj}/V_j^0\in\{0,1/5\}$ and $V_j^0\ge3\gamma W$ when $\gamma\le u/2$. For the illustrative choice $u=0.05$, $\gamma=0.0125$, these quantities are $V_j^0=0.3125$ and $3\gamma W=0.1875$. The perturbation reduces the conditional tail probability from $0.0625$ to $0.0375$ on router $j$'s chosen action. Router $j$ therefore has tail probability $0.0375$, and every other router has tail probability $0.0575$ or $0.0625$. These are exact feasibility calculations, not experimental observations.
\section{Additional certificate constructions}\label{app:certificates}
These supporting constructions complement the recovery workflow in Section~\ref{sec:repair}. They cover learned policy mixtures, the CDF and betting implementations used in the experiments, and valid use of normalization information; the main information-separation result does not require all of them.

\subsection{Posterior-uniform exceedance certificates}\label{app:posterior_certificate}
This extends the fixed-candidate setting of Section~\ref{sec:repair} to a policy drawn from a data-selected distribution over candidates; the guarantee concerns the resulting pooled outcome law.

For this appendix, replace the finite list by a family $\{\pi_\theta:\theta\in\Theta\}$ fixed independently of the log. Require $\pi_\theta(\cdot\mid x)\ll\pi_0(\cdot\mid x)$ for every $x,\theta$, and Assumption~\ref{ass:overlap} with $j$ replaced by every $\theta\in\Theta$. Define $w_\theta$, $G_\theta$, and $\widehat G_\theta$ by the formulas in Section~\ref{sec:setup} with the same replacement, and let $Q_p(\theta)=\inf\{t:G_\theta(t)\le1-p\}$.

A prior $\pi$ over $\Theta$ is fixed independently of $\calD_n$; a posterior $\rho\ll\pi$ may depend on all of $\calD_n$. The deployed randomized policy draws $\theta\sim\rho$ per decision, so its outcome law has exceedance, CDF, and quantile
\begin{align*}
 \bar G_\rho(t)&=\E_{\theta\sim\rho}G_\theta(t),\\
 \bar F_\rho(t)&=1-\bar G_\rho(t),\\
 \bar Q_p(\rho)&=\inf\big\{t:\bar G_\rho(t)\le 1-p\big\},
\end{align*}
and we write $\widehat G_\rho(t)=\E_{\theta\sim\rho}\widehat G_\theta(t)$. For a finite list $\Theta=\{\theta_1,\ldots,\theta_K\}$ with positive prior masses, a point posterior $\rho=\delta_\theta$ recovers a deterministic choice among candidates, with $\KL(\delta_{\theta_j}\|\pi)=\log(1/\pi(\theta_j))$, which equals $\log K$ under a uniform prior over $K$ candidates. Reporting uses a nonempty finite predeclared grid $\mathcal T\subset[0,B)$; the trivial bound $\bar Q_p(\rho)\le B$ is always available.

\paragraph{Why the pooled quantile.} $\bar Q_p(\rho)$ is the quantile of the outcomes the operator actually observes when the randomized policy is deployed; it is the natural target of a service-level constraint. The posterior average of per-policy quantiles, $\E_\rho Q_p(\theta)$, is a different object because quantiles are nonlinear; we do not target it.

\begin{theorem}[Selection-valid exceedance envelope]\label{thm:exc}
Under Assumption~\ref{ass:overlap}, let $\Lambda\subset(0,2/W)$ be a nonempty finite predeclared set. With probability at least $1-\delta$, simultaneously for every data-dependent posterior $\rho\ll\pi$, every $t\in\mathcal T$ and every $\lambda\in\Lambda$,
\begin{equation}\label{eq:exc}
 \bar G_\rho(t)\le
 U_{\rho,\lambda}(t):=
 \frac{\widehat G_\rho(t)+\dfrac{\KL(\rho \,\|\, \pi)+\log\frac{|\mathcal T|\,|\Lambda|}{\delta}}{\lambda n}}{1-\lambda W/2}.
\end{equation}
Let $U_\rho(t)=\min_{\lambda\in\Lambda}U_{\rho,\lambda}(t)$. Then, on the same event,
\begin{equation}\label{eq:quant}
 \bar Q_p(\rho)\le \widehat q^{\,\rm cert}_p(\rho):=\min\big\{t\in\mathcal T: U_\rho(t)\le 1-p\big\},
\end{equation}
whenever the set is nonempty (otherwise the certificate is \emph{inconclusive}).
\end{theorem}

The statistic $\widehat G$ is the complementary-CDF estimator of \citet{huang2022offarxiv}, who bound it for a fixed policy (their Lemma~4.4). Theorem~\ref{thm:exc} makes the bound uniform over data-dependent posteriors. It follows from the one-sided bound for nonnegative variables \citep{maurer2003bound} in its PAC-Bayes form \citep{thiemann2017strongly,gabbianelli2024importance}; Appendix~\ref{app:exc} gives a short proof. The summand $w_\theta\1\{Y>t\}$ is nonnegative, so its lower tail needs no range condition, and its second moment is at most $W\,G_\theta(t)$. The second-moment proxy is therefore linear in $G_\theta$ and averages to the \emph{pooled} exceedance. Continuous optimization in $\lambda$, or a sufficiently rich predeclared grid up to constants, gives the ``fast-rate'' form
\begin{equation}\label{eq:scale}
 \bar G_\rho(t)\lesssim \widehat G_\rho(t)+\sqrt{\frac{2W\widehat G_\rho(t)\,C_\rho}{n}}+\frac{2WC_\rho}{n},
\end{equation}
with $C_\rho=\KL(\rho \,\|\, \pi)+\log(|\mathcal T|\,|\Lambda|/\delta)$. Near the target $\widehat G_\rho\approx1-p$, so the radius scales with the tail mass rather than a constant.

\paragraph{Properties.} The certificate reads off a crossing of the exceedance envelope, so it needs no density or inverse-slope condition and remains meaningful at atoms. Because~\eqref{eq:exc} is simultaneous over $\rho$, the posterior may be chosen by any rule on the same data, with no holdout. A known tail-specific weight bound can sharpen this expression (Remark~\ref{rem:tailW}); otherwise the global $W$ still controls it.

\paragraph{Other certificate choices.}
We call this envelope \emph{Exc-PB}. Appendix~\ref{app:betting} gives \emph{Exc-Bet}, a variance-adaptive betting certificate for a fixed candidate list with no $\lambda$-grid cost; its reported power scaling is empirical. Proposition~\ref{prop:pbpower} proves a worst-case rate for Exc-PB with its tuning-grid cost, while the matched rate without that cost uses established Exc-EB. Appendix~\ref{app:cdf} gives the complementary CDF-PB construction. Its bulk second moment can be much larger than the tail second moment; Section~\ref{sec:boundary} identifies an observation model in which this information difference has an unavoidable cost.


\subsection{Proof of Theorem~\ref{thm:exc}}\label{app:exc}
The proof makes explicit the two ingredients behind the preceding envelope: a tail-sensitive exponential bound and a change of measure that permits same-data policy selection.

Fix $\theta$, $t$ and $\lambda\in(0,2/W)$, and let $U_i=w_\theta(X_i,A_i)\1\{Y_i>t\}\ge0$ with $\E U_i=G_\theta(t)$. Since $e^{-x}\le1-x+x^2/2$ for $x\ge0$,
\[
 \E e^{-\lambda U_i}\le 1-\lambda\E U_i+\tfrac{\lambda^2}{2}\E U_i^2\le\exp\Bigl(-\lambda\E U_i+\tfrac{\lambda^2}{2}\E U_i^2\Bigr).
\]
Assumption~\ref{ass:overlap} gives $\E U_i^2\le W\E U_i=WG_\theta(t)$. By independence,
\[
 \E\exp\Bigl\{\lambda n\bigl(G_\theta(t)-\widehat G_\theta(t)\bigr)-\tfrac{\lambda^2nW}{2}G_\theta(t)\Bigr\}\le1.
\]
Denote the exponent by $f_{t,\lambda}(\theta)$. Integrating over $\pi$ (Fubini; $\pi$ is independent of the data), applying Markov's inequality and a union bound over $\mathcal T\times\Lambda$, we obtain an event of probability at least $1-\delta$ on which $\E_{\theta\sim\pi}e^{f_{t,\lambda}(\theta)}\le|\mathcal T||\Lambda|/\delta$ for all $(t,\lambda)$. On this event, the Donsker--Varadhan inequality gives, for every $\rho\ll\pi$,
\[
 \lambda n\Bigl(\bar G_\rho(t)-\widehat G_\rho(t)\Bigr)-\tfrac{\lambda^2nW}{2}\bar G_\rho(t)\le\KL(\rho \,\|\, \pi)+\log\tfrac{|\mathcal T|\,|\Lambda|}{\delta},
\]
using linearity of $\theta\mapsto G_\theta(t)$ under $\E_\rho$. Rearranging and dividing by $\lambda n(1-\lambda W/2)>0$ gives~\eqref{eq:exc}. For~\eqref{eq:quant}, if $t^\ast=\widehat q_p^{\rm cert}(\rho)$ exists, then $\bar G_\rho(t^\ast)\le U_\rho(t^\ast)\le1-p$, so $\bar F_\rho(t^\ast)\ge p$ and $\bar Q_p(\rho)\le t^\ast$.

\begin{remark}[Tail-local overlap]\label{rem:tailW}
The proof uses Assumption~\ref{ass:overlap} only through $\E[w_\theta^2\1\{Y>t\}]\le W\,G_\theta(t)$. Any known bound $0\le W_t\le W$ with $w_\theta(X,A)\le W_t$ almost surely on $\{Y>t\}$ for all $\theta$ may replace $W$ in~\eqref{eq:exc}, retaining the same paid grid $\Lambda$. A larger proposed bound can first be replaced by its minimum with $W$. This improvement requires the tail-local bound to be known independently of the certification sample. Without it, using the global $W$ can still make large weights outside the tail costly through the bound or its range term.
\end{remark}

\subsection{Worst-case power of the PAC-Bayes envelope}\label{app:pbpower}
This result quantifies the power of Exc-PB as a supporting certificate for Section~\ref{sec:experiments}; it includes the tuning-grid cost and is separate from the matched Exc-EB rate in Theorem~\ref{thm:boundary}.

\begin{proposition}[Finite-list PAC-Bayes power]\label{prop:pbpower}
Let $K\ge1$ candidates satisfy Assumption~\ref{ass:overlap}, let $p\in(0,1)$, $u=1-p$, $0<\gamma\le u/2$, and $\delta,\beta\in(0,1/2)$. Use a uniform prior, $\mathcal T=\{\tau\}$, and the grid
\[
 \Lambda=\Big\{2^{-k}/(2W):k=0,\ldots,\lceil\log_2n\rceil\Big\}.
\]
Write $G_j=G_j(\tau)$ and $\widehat G_j=\widehat G_j(\tau)$ in this proposition, and set $C^*=\max\{\log(K|\Lambda|/\delta),\log(1/\beta)\}$. The rule selecting any candidate certified by Theorem~\ref{thm:exc} is $\delta$-valid. If $s$ is uniquely safe and $G_s\le u-\gamma$, it outputs $s$ with probability at least $1-\beta-\delta$ whenever
\begin{equation}\label{eq:pbupper}
 n\ge\frac{41WuC^*}{\gamma^2}.
\end{equation}
\end{proposition}
\begin{proof}
Validity follows from Theorem~\ref{thm:exc}. For power, put $g=G_s$ and $\lambda_0=\gamma/(2Wu)\le1/(4W)$. Condition~\eqref{eq:pbupper} implies $n\ge u/\gamma$, because $C^*\ge\log2$, $W\ge1$, and $\gamma\le1/2$. Hence the smallest grid point, at most $1/(2Wn)$, is no larger than $\lambda_0$. Choose $\lambda\in\Lambda$ with $\lambda_0/2<\lambda\le\lambda_0$. Bernstein's inequality gives, with probability at least $1-\beta$,
\[
 \widehat G_s-g
 \le\sqrt{\frac{2Wu\log(1/\beta)}n}
       +\frac{W\log(1/\beta)}{3n}.
\]
Writing $C=\log(K|\Lambda|/\delta)\le C^*$, it follows that
\begin{align*}
 (\widehat G_s-g)+\frac{C}{\lambda n}+\frac{u\lambda W}{2}
 &\le\gamma\left(\sqrt{\frac2{41}}
            +\frac1{246}+\frac4{41}+\frac14\right)\\
 &<\gamma.
\end{align*}
Here the four terms use, respectively,~\eqref{eq:pbupper}, $\gamma/u\le1/2$, $\lambda>\gamma/(4Wu)$, and $\lambda\le\gamma/(2Wu)$. Since $g\le u-\gamma$, this inequality implies
\[
 \widehat G_s+\frac{C}{\lambda n}
 \le u(1-\lambda W/2),
\]
so $U_{\delta_{\theta_s},\lambda}(\tau)\le u$ and $s$ is certified. On the simultaneous-validity event every certified candidate is safe; uniqueness therefore forces output $s$. The intersection has probability at least $1-\beta-\delta$.
\end{proof}

This is a worst-case $Wu$ bound with the logarithmic cost of the predeclared grid. It does not establish a rate adaptive to the tail second moment for arbitrary learned posteriors. To target output power exactly $1-\beta$, one may run the result with validity level $\min(\delta,\beta/2)$ and power-failure parameter $\beta/2$.

\subsection{TailCert: the CDF-envelope certificate}\label{app:cdf}
This supplies the TailCert construction described in Section~\ref{sec:repair}, called CDF-PB in the experimental comparisons. It uses the independently fixed policy family, prior $\pi$, bounded weights, pooled CDF, and admissible posteriors $\rho\ll\pi$ defined in Appendix~\ref{app:posterior_certificate}. The posterior expectation must be evaluated exactly or accompanied by a valid approximation-error bound; an unprotected finite Monte Carlo approximation is not covered.

Write $\widehat F_\theta(t)=n^{-1}\sum_iw_\theta(X_i,A_i)\1\{Y_i\le t\}$ and $\widehat F_\rho(t)=\E_{\theta\sim\rho}\widehat F_\theta(t)$. In general,
\[
 \widehat F_\rho(t)+\widehat G_\rho(t)
 =\E_\rho\frac1n\sum_iw_\theta(X_i,A_i)\ne1,
\]
so complementing the empirical CDF is not the same calculation as directly estimating exceedance.

For a fixed threshold, write $F_\theta(t)=1-G_\theta(t)$, $\bar F_\rho(t)=\E_{\theta\sim\rho}F_\theta(t)$, and $V_F(\theta,t)=\E[w_\theta^2\1\{Y\le t\}]\le WF_\theta(t)\le W$. To distinguish its tuning range from Exc-PB, let $\Lambda_F\subset(0,3/W)$ be a nonempty finite predeclared grid and put $C^F_\rho=\KL(\rho\|\pi)+\log(|\mathcal T||\Lambda_F|/\delta)$. A Bernstein exponential bound and the PAC-Bayes change of measure give, simultaneously over $t\in\mathcal T$, $\lambda\in\Lambda_F$, and admissible posteriors,
\begin{align*}
 \bar F_\rho(t)&\ge\widehat F_\rho(t)-r_\lambda(\rho,t),\\
 r_\lambda(\rho,t)&=\frac{C^F_\rho}{\lambda n}
       +\frac{\lambda\E_\rho V_F(\theta,t)}{2(1-W\lambda/3)},
 \qquad \lambda\in\Lambda_F,
\end{align*}
This display uses a population second moment: an executable certificate must use an independently justified bound (for example $V_F\le W$), or a valid empirical-variance construction; substituting an unprotected sample estimate is not licensed by the argument. On the rare-action model, $V_F(\theta_j,\tau)=WF_{\theta_j}(\tau)$, of order $Wp$ near the threshold. Theorem~\ref{thm:boundary}(b) applies to procedures whose entire observation is the joint CDF sequence at that threshold. A rule that additionally observes weights, action counts, the endpoint CDF at $B$, or full-log selection scores is outside that restriction. The lower bound is not a universal claim about all CDF envelopes.
A fully executable, possibly conservative alternative follows by using $\E_\rho V_F\le W\bar F_\rho$ and rearranging:
\[
 \bar F_\rho(t)\ge
 \frac{\widehat F_\rho(t)-C^F_\rho/(\lambda n)}
 {1+\lambda W/\{2(1-W\lambda/3)\}}.
\]
Clip the lower bound at zero and optimize only over the paid predeclared grid. The CDF-PB comparator in our experiments uses the first display with the conservative proxy $\E_\rho V_F\le W$, minimized over a predeclared grid of $60$ values of $\lambda$ whose size is paid in $C^F_\rho$ (\texttt{common/certs.py}).

\paragraph{Executable envelope and quantile inversion.}
With $[x]_0^1=\min\{1,\max\{0,x\}\}$, the conservative proxy $\E_\rho V_F\le W$ gives
\begin{equation}\label{eq:tailcertcdf}
 L_\rho(t)=\left[\max_{\lambda\in\Lambda_F}
 \left\{\widehat F_\rho(t)-\frac{C^F_\rho}{\lambda n}
 -\frac{\lambda W}{2(1-\lambda W/3)}\right\}\right]_0^1.
\end{equation}
An infinite KL makes the bound vacuous. For finite KL, $\bar F_\rho(t)\ge L_\rho(t)$ holds with probability at least $1-\delta$, simultaneously over $t\in\mathcal T$ and data-dependent $\rho$. At point posteriors, this is the envelope used in our experiments by \texttt{common/certs.py::cdf\_pb}, which returns $1-L_\rho(t)$ to express every experimental comparison in exceedance units.

For fixed $t\in\mathcal T$ and $\lambda\in\Lambda_F$, the Bernstein exponential inequality gives
\[
 \E_{\calD_n}\exp\left\{
 \lambda n \Big(\widehat F_\theta(t)-F_\theta(t)\Big)
 -\frac{n\lambda^2 V_F(\theta,t)}{2(1-\lambda W/3)}
 \right\}\le1.
\]
Integrate over the independent prior and apply Markov's inequality, allocating error $\delta/(|\mathcal T||\Lambda_F|)$ to each grid pair. On the resulting common event, the change-of-measure inequality gives, for every posterior, the preceding CDF bound with $C^F_\rho$. Replacing $\E_\rho V_F$ by $W$, maximizing over the paid grid, and clipping to $[0,1]$ proves~\eqref{eq:tailcertcdf}. This argument also justifies the general second-moment expression above.

TailCert takes the first grid crossing of $p$. On the common event,
$\bar F_\rho(\widehat q^{\rm TailCert}_p(\rho))\ge L_\rho(\widehat q^{\rm TailCert}_p(\rho))\ge p$,
so $\bar Q_p(\rho)\le\widehat q^{\rm TailCert}_p(\rho)$. The implication uses the definition of a quantile, with no density or strict-monotonicity condition, and therefore remains valid at atoms. If no grid point crosses, the nontrivial certificate is inconclusive; the known support bound $B$ remains available. At a predeclared budget $\tau\in\mathcal T$, selecting any posterior satisfying $L_\rho(\tau)\ge p$ preserves the same guarantee. The boundary in Theorem~\ref{thm:boundary} concerns the finite-class selection problem and uses empirical Bernstein for its matching achievable rates; it is not a sharp-rate theorem for this PAC-Bayes construction.

\subsection{The betting certificate}\label{app:betting}
This gives the implementation and validity argument for Exc-Bet in Section~\ref{sec:experiments}, followed by the stratified construction used for held-out benchmark evaluation.

Fix a candidate $j$, suppress its label by writing $w=w_j$ and $G(t)=G_j(t)$, and define the normalized exceedance observations
\[
 R_i=w(X_i,A_i)\1\{Y_i>\tau\}/W\in[0,1],
 \qquad \mu=\E R_i=G(\tau)/W.
\]
The $R_i$ are i.i.d. The test bets against the candidate being unsafe: for $m\in[0,1]$, predictable bets $\lambda_i\in[0,1/2]$ give wealth $\mathcal W_n(m)=\prod_{i=1}^n\{1+\lambda_i(m-R_i)\}$, a nonnegative supermartingale under the null $\mu\ge m$. Certifying at $m=(1-p)/W$ when wealth reaches $K/\delta$ is therefore valid after a union over candidates. The implementation adapts the predictable plug-in strategy of \citet{waudby2024estimating}, holding bets fixed within doubling blocks and inverting a quadratic lower bound in closed form.

For the predeclared horizon $n$, partition $\{1,\ldots,n\}$ into consecutive blocks $I_b$ of nominal sizes $64\cdot2^{b-1}$, truncating the final block at $n$. Define
\[
 n_b=|I_b|,\quad c_b=\sum_{h<b}n_h,\quad
 S_{1b}=\sum_{i\in I_b}R_i,\quad S_{2b}=\sum_{i\in I_b}R_i^2,
\]
and let $S_{r,<b}=\sum_{h<b}S_{rh}$ for $r\in\{1,2\}$, with empty sums equal to zero. Before block $b$, compute the regularized variance estimate
\[
 v_b=\frac{1/4+\max\Big\{S_{2,<b}-S_{1,<b}^2/\max(c_b,1),0\Big\}}{c_b+1}.
\]
Thus $v_1=1/4$. Set $\lambda_b=\min\big\{\sqrt{2\log(K/\delta)/(v_bn)},1/2\big\}$ and $\lambda_i=\lambda_b$ for $i\in I_b$. For any $m\in[0,1]$ and $y=\lambda_b(m-R_i)\ge-1/2$,
\begin{align*}
 \log\big\{1+\lambda_b(m-R_i)\big\}
 &\ge \lambda_b(m-R_i)+\phi\,\lambda_b^2(m-R_i)^2,
\end{align*}
where $\phi=4(\log\tfrac12+\tfrac12)<0$. Consequently,
\begin{align*}
 \log\mathcal W_n(m)&\ge Q(m),\\
 Q(m)&:=\sum_b\Bigl\{\lambda_b\big(n_bm-S_{1b}\big)+\phi\lambda_b^2\big(n_bm^2-2mS_{1b}+S_{2b}\big)\Bigr\},
\end{align*}
$Q$ is a concave quadratic. Each observation contributes derivative
\[
 \lambda_b\big\{1+2\phi\lambda_b(m-R_i)\big\}
 \ge\lambda_b(1+\phi)>0
\]
on $[0,1]$, using $\lambda_b\le1/2$ and $\phi>-1$. Thus $Q$ is increasing there. Define $\hat m=\min\{m\in[0,1]:Q(m)\ge\log(K/\delta)\}$ (or $1$ if none) and the bound $U=\min\{1,W\hat m\}$.

\emph{Validity.} For the predeclared horizon $n$, the bets do not depend on $m$ and are predictable, and $1+\lambda_b(\mu-R_i)\ge1/2>0$. Hence $\mathcal W_n(\mu)$ is a nonnegative martingale with unit mean. If $G(\tau)>U$, then $G(\tau)\le1$ forces $U<1$, hence $U=W\hat m$ and $\mu>\hat m$, and monotonicity gives $\log\mathcal W_n(\mu)\ge Q(\mu)\ge\log(K/\delta)$. Ville's inequality bounds the probability of this event by $\delta/K$, and a union bound over the $K$ candidates makes the selected candidate's certificate valid. The certificate is valid for a fixed candidate list; for data-dependent posteriors we use Theorem~\ref{thm:exc}.

\paragraph{Stratified version: construction.}
Suppose two independent strata are each i.i.d.\ with fixed sizes and the same known weight bound $W$. Their exceedance means are $G^{(1)},G^{(2)}$, and the target is $G=qG^{(1)}+(1-q)G^{(2)}$ for a declared $q\in(0,1)$. Run the preceding test within each stratum, obtaining quadratic lower bounds $Q_1,Q_2$. At the true means, independence and the unit-mean wealths imply
\[
 \PP\!\left\{Q_1(G^{(1)}/W)+Q_2(G^{(2)}/W)
                  \ge\log(K/\delta)\right\}\le\delta/K.
\]
For $g\in[0,W]$, define
\[
 H(g)=\min_{\substack{m_1,m_2\in[0,1]\\qWm_1+(1-q)Wm_2=g}}
           \{Q_1(m_1)+Q_2(m_2)\}.
\]
Let $U$ be the smallest $g$ with $H(g)\ge\log(K/\delta)$, or $W$ if no crossing exists.

\paragraph{Exact computation.}
For $x=g/W$, the feasible interval for $m_1$ is $[a,b]$, with $a=\max\{0,(x-(1-q))/q\}$ and $b=\min\{1,x/q\}$. The objective $Q_1(m_1)+Q_2((x-qm_1)/(1-q))$ is concave, so its minimum is the smaller endpoint value, a continuous function of $g$. A discrete-grid minimum would not be conservative. Moreover, any feasible pair at a larger $g$ can be reduced coordinatewise to a feasible pair at a smaller $g$. Since each $Q_s$ is increasing, $H$ is nondecreasing, allowing bisection for $U$.

\paragraph{Validity after selection.}
If $G>U$, a crossing exists because otherwise $U=W\ge G$. Monotonicity then gives $H(G)\ge\log(K/\delta)$. The true split is feasible at $G$, so the event in the probability display occurs. Thus $\PP(G>U)\le\delta/K$; a union over candidates gives selection validity. Returning $\min\{1,U\}$ preserves this guarantee because $G\le1$. This uses one confidence budget for both strata. It is tighter than separate stratum-wise bounds in the reported deployment comparison; no pointwise ordering is asserted.

\subsection{Normalization information in the tail}\label{app:sn}
This explains why the recovery discussion in Section~\ref{sec:repair} calls for a normalization-aware baseline: retained weight information can improve inference, and the CDF-only lower bound does not compare all such methods.

For a fixed candidate $j$, write $w=w_j(X,A)$, $w_i=w_j(X_i,A_i)$, $Z=\1\{Y>\tau\}$, and $Z_i=\1\{Y_i>\tau\}$. Let $G=\E[wZ]$ and $V=\E[w^2Z]$; support coverage gives $\E w=1$. The unnormalized estimate is $\widehat G=n^{-1}\sum_iw_iZ_i$, and the self-normalized estimate is $\widehat G_{\rm SN}=\sum_iw_iZ_i/\sum_iw_i$ whenever its denominator is positive. For fixed distributions with bounded weights, $\sigma_E^2=n\operatorname{Var}(\widehat G)$ and the asymptotic variance coefficient $\sigma_{\rm SN}^2$ of $\sqrt n(\widehat G_{\rm SN}-G)$ are
\begin{align*}
 \sigma_E^2&=V-G^2,\\
 \sigma_{\rm SN}^2&=\E \big[w^2(Z-G)^2 \big]
                 =V(1-2G)+G^2\E w^2.
\end{align*}
Thus both are asymptotic to $V$ along a small-tail sequence only under additional conditions, for example $G^2/V\to0$ and $G^2\E w^2/V\to0$. Small $G$ alone is insufficient. With two actions, logging $(1/2,1/2)$, target $(\varepsilon,1-\varepsilon)$ for $\varepsilon\in(0,1)$, and deterministic tails $(1,0)$, weights are bounded by two and
\[
 G=\varepsilon,\quad V=2\varepsilon^2,\quad
 \sigma_E^2=\varepsilon^2,\quad
 \sigma_{\rm SN}^2=4\varepsilon^2(1-\varepsilon)^2.
\]
The ratio $\sigma_{\rm SN}^2/\sigma_E^2$ tends to four as $\varepsilon\downarrow0$. Self-normalization may also improve constants on other instances. The full-data lower bounds delimit possible rates on their specified families; they do not rank all normalized and unnormalized procedures pointwise.

Whenever the denominator is positive, the self-normalized CDF and exceedance estimates sum to one. They retain the weight normalizer and therefore lie outside the joint-CDF-only model. Finite-sample normalized inference need not pay for separate numerator and denominator intervals: joint betting or control-variate methods can use $\E w=1$ directly \citep{karampatziakis2021offpolicy,kuzborskij2021confident}.

\paragraph{A valid normalization-aware comparator.}
For a fixed candidate and null $G\ge m$, with $m\in[0,1]$, choose predictable bets $0\le\lambda_i\le1/(2W)$ and form
\[
 E_n(m)=\prod_{i=1}^n\big\{1+\lambda_iw_i(m-Z_i)\big\}.
\]
Every factor is at least $1/2$, and its conditional expectation under the null is $1+\lambda_i(m-G)\le1$. Hence $E_n(m)$ is a nonnegative supermartingale starting at one. At the unscaled tail budget $m=u=1-p$, certifying a candidate when $E_n(u)\ge K/\delta$ is valid after a union bound over a fixed list of $K$ candidates; Ville's inequality also permits monitoring over time when the predictable strategy is fixed consistently. A predeclared mixture of such wealth processes is valid, while post-hoc maximization across strategies requires its own error allocation. This is an established style of off-policy inference, offered as a comparator rather than a new optimality result. Our implementation (Norm-Bet) holds bets fixed within doubling blocks, chooses them from earlier blocks by a plug-in variance at the running estimate of $G$, and inverts a quadratic lower bound on the log-wealth exactly as in Appendix~\ref{app:betting}; the bound is increasing in $m$ because $\lambda_i w_i\le1/2$. On the rare-action family it matches Exc-Bet within $9\%$ (Figure~\ref{fig:information}).

\section{Logging design and scope limits}\label{app:design_scope}
This section explains two boundaries of the main message: how logging design changes tail overlap, and why overlap or normalization alone does not remove every cost of data-dependent selection.

\subsection{Tail-aware logging design}\label{app:loggingdesign}
This records the classical importance-sampling optimization for the tail second moment, complementing the overlap interpretation in the main text. It concerns a fixed target policy, rather than joint logging design for an unknown selected candidate.

Tail-Influence Sampling \citep{bourigault2026tail} uses a related variance-allocation principle for fixed-policy CVaR estimation through fresh conditional queries. Here we minimize the second moment of a weighted exceedance under action logging; the objectives and sampling models differ.

\begin{proposition}[Tail-aware logging]\label{prop:logging}
Fix a target policy $\pi_{\rm tar}$, a threshold $\tau$, conditional tail probabilities $g(x,a)=\PP(Y>\tau\mid X=x,A=a)$, and $M$ actions indexed by $a\in\{1,\ldots,M\}$. For a logger $\pi_0$, write $w(x,a)=\pi_{\rm tar}(a\mid x)/\pi_0(a\mid x)$ on its support; in the ratio sums below, a summand with $\pi_{\rm tar}(a\mid x)=\pi_0(a\mid x)=0$ is defined as zero. The infimum of
\[
 V_\tau(\pi_0)=\E_X\sum_a
 \frac{\pi_{\rm tar}(a\mid X)^2g(X,a)}{\pi_0(a\mid X)}
\]
over logging policies dominating $\pi_{\rm tar}$ is
\[
 V_\tau^\star=\E_X\left(\sum_a\pi_{\rm tar}(a\mid X)\sqrt{g(X,a)}\right)^2.
\]
Where the normalizer is positive, the formal minimizer is $\pi_0^\star(a\mid x)\propto\pi_{\rm tar}(a\mid x)\sqrt{g(x,a)}$. It attains the infimum when it dominates $\pi_{\rm tar}$; in general the infimum is approached by mixtures with uniform exploration. On contexts where all target-supported tail probabilities are zero, any admissible logger has zero contribution. Relative to target-policy logging,
\[
 V_\tau(\pi_{\rm tar})-V_\tau^\star
 =\E_X\operatorname{Var}_{a\sim\pi_{\rm tar}(\cdot\mid X)}
       [\sqrt{g(X,a)}].
\]
If $\kappa\ge1$ and an estimate satisfies $g/\kappa^2\le\hat g\le\kappa^2g$ on the target support, set $\widehat\pi_0^\star\propto\pi_{\rm tar}\sqrt{\hat g}$, choosing any distribution when its normalizer is zero. For $\varepsilon\in(0,1)$, the logger $\pi_0=(1-\varepsilon)\widehat\pi_0^\star+\varepsilon/M$ satisfies
\[
 V_\tau(\pi_0)\le\frac{\kappa^2}{1-\varepsilon}V_\tau^\star,
 \qquad w\le M/\varepsilon.
\]
\end{proposition}
\begin{proof}
For each context, Cauchy--Schwarz yields
\[
 \sum_a\frac{\pi_{\rm tar}^2g}{\pi_0}
 \ge\left(\sum_a\pi_{\rm tar}\sqrt g\right)^2.
\]
Equality holds at the stated formal proposal on positive-tail support. Choose it arbitrarily on zero-normalizer contexts and set $\pi_{0,\varepsilon}=(1-\varepsilon)\pi_0^\star+\varepsilon/M$. The lower bound above and $\pi_{0,\varepsilon}\ge(1-\varepsilon)\pi_0^\star$ on positive-tail support imply
\[
 V_\tau^\star\le V_\tau(\pi_{0,\varepsilon})
                   \le\frac{V_\tau^\star}{1-\varepsilon}.
\]
Zero-tail contributions vanish. Letting $\varepsilon\downarrow0$ proves the claimed infimum. Under target-policy logging, the corresponding conditional moment is $\sum_a\pi_{\rm tar} g$; subtracting $(\sum_a\pi_{\rm tar}\sqrt g)^2$ gives the displayed variance identity.

For robustness, let $H(x)=\sum_a\pi_{\rm tar}(a\mid x)\sqrt{\hat g(x,a)}$. On positive-tail support,
\[
 \frac{\pi_{\rm tar}^2g}{\pi_0}
 \le\frac{H\pi_{\rm tar} g}{(1-\varepsilon)\sqrt{\hat g}}.
\]
The assumed accuracy gives $g/\sqrt{\hat g}\le\kappa\sqrt g$ and $H\le\kappa\sum_a\pi_{\rm tar}\sqrt g$. Summation and expectation prove the bound. If $H=0$, the accuracy assumption implies $g=0$ on the target support, so that context contributes zero. Finally, $\pi_0\ge\varepsilon/M$ implies $w\le M/\varepsilon$.
\end{proof}

The exact improvement concerns $V_\tau$. For a deterministic target it is zero relative to target-policy logging. A positive moment improvement need not yield the same sample-size improvement: Theorem~\ref{thm:instance}(a) retains a range term, and the certificate must adapt to the realized variance. The multiplicative model-accuracy condition is an assumption, not a guarantee for an arbitrary fitted tail model. Any logging design used for certification must be fixed independently of that certification sample, or analyzed under a suitable adaptive extension.

\subsection{Mixture overlap does not suffice after selection}\label{app:mixture}
This counterexample explains a limitation of extending Section~\ref{sec:repair} to learned mixtures: small overlap for the selected mixture alone does not control the cost of selecting that mixture on the same log.

For a fixed posterior, the deployed mixture has weight $w_\rho(x,a)=\E_{\theta\sim\rho}w_\theta(x,a)$ and $V_{\rm mix}(\rho)=\E_{\pi_0}[w_\rho^2]\le\E_\rho V_w(\theta)$, where $V_w(\theta)=\E_{\pi_0}[w_\theta^2]$. The mixture second moment can be much smaller than $\E_\rho V_w(\theta)$.

\begin{proposition}\label{prop:mixture}
Fix $0<q<1$, confidence $\delta\in(0,1)$, and nonzero grid cardinalities $|\mathcal T|,|\Lambda|$ that remain fixed as $n=K\to\infty$, so $C_\rho=\KL(\rho\|\pi)+\log(|\mathcal T||\Lambda|/\delta)$ has a fixed non-KL term. In the spirit of propensity overfitting \citep{swaminathan2015self}, for each $n=K$, there is a one-context problem with uniform logging over $K$ actions, a uniform prior $\pi$ over their deterministic policies, and a data-selected posterior $\rho_n$ such that $\bar F_{\rho_n}(0)=q$ for every dataset, while with probability tending to one $\widehat F_{\rho_n}(0)=1$ and $\frac1n\sum_iw_{\rho_n}(A_i)=1$. Moreover, in probability, $V_{\rm mix}(\rho_n)=\|w_{\rho_n}\|_\infty\to e/q$ and $\KL(\rho_n\|\pi)\to\log(e/q)$. Thus no universal-constant bound on $\widehat F_\rho(0)-\bar F_\rho(0)$ of order $\sqrt{V_{\rm mix}(\rho)C_\rho/n}+\|w_\rho\|_\infty C_\rho/n$ can hold uniformly over selected posteriors.
\end{proposition}

\begin{proof}
Let $A_i$ be uniform on $\{1,\dots,K\}$ and $B_i\sim\mathrm{Bernoulli}(q)$ independent, $Y_i=1-B_i$. Let $N_j=\sum_i\1\{A_i=j\}$, $S=\{j:N_j=1,\text{ the unique }B_i=1\}$, $M_S=|S|$. If $M_S>0$ let $\rho_n$ be uniform on $S$; otherwise $\rho_n=\pi$. Every candidate has $F_j(0)=q$. For $M_S>0$, $w_{\rho_n}(a)=(K/M_S)\1\{a\in S\}$, exactly $M_S$ logged points have selected actions and all have $B_i=1$, so $\widehat F_{\rho_n}(0)=\frac{M_S}{K}\cdot\frac{K}{M_S}=1$, $V_{\rm mix}(\rho_n)=\|w_{\rho_n}\|_\infty=K/M_S$ and $\KL(\rho_n\|\pi)=\log(K/M_S)$. 

For the occupancy step, put $I_j=\1\{j\in S\}$. For $j\ne k$ and $K\ge3$,
\[
 \E I_j=q(1-1/K)^{K-1},\qquad
 \E[I_jI_k]=q^2\frac{K-1}{K}(1-2/K)^{K-2}.
\]
The second identity assigns the two singleton actions to distinct observations; all other observations avoid both actions. Consequently,
\[
 \E[(M_S/K)^2]
 =\frac qK(1-1/K)^{K-1}
   +q^2\left(\frac{K-1}{K}\right)^2(1-2/K)^{K-2}
 \longrightarrow q^2/e^2.
\]
Together with $\E[M_S/K]\to q/e$, this gives vanishing variance and $M_S/K\to q/e$ in probability. Thus $\PP(M_S>0)\to1$. The proposed bound's right-hand side tends to zero in probability, while the estimation error equals $1-q>0$ with probability tending to one. For any fixed multiplicative constant its failure probability therefore tends to one, contradicting a uniform $1-\delta$ guarantee.
\end{proof}

\subsection{Scope of the comparison}\label{app:positioning}
Table~\ref{tab:positioning} complements Section~\ref{sec:related} by separating the information comparison proved here from the existing inference tools used to attain its rates.

\begin{table}[!htbp]
\centering\small
\caption{\textbf{Where our contribution lies.}}\label{tab:positioning}
\begin{tabularx}{\textwidth}{p{0.23\textwidth}XX}
\toprule
\rowcolor{tablehead}
\textbf{Prior tool or result} &\textbf{What is already available} &\textbf{What the present comparison adds} \\
\midrule
CCDF inference \citep{huang2022offarxiv} & Tail-sensitive exceedance estimates, empirical-Bernstein bounds, and CDF/CCDF combination. & \textbf{Matched selection-and-certification costs} for full logs versus the joint CDF reduction.\\
Normalized off-policy inference \citep{kuzborskij2021confident,karampatziakis2021offpolicy} & Valid use of weight normalization and policy-selection inference. & An explicit binary-model \textbf{sufficiency and recovery explanation}; no universal advantage over normalized methods is claimed.\\
PAC-Bayes nonnegative-loss bounds \citep{gabbianelli2024importance,thiemann2017strongly} & Posterior-uniform concentration and small-mean adaptation. & Application to the pooled tail constraint, used as \textbf{supporting machinery}.\\
Off-policy lower bounds \citep{li2015toward,huang2022offarxiv,wan2026riskaware} & Lower bounds for their estimation or learning targets. & \textbf{Tail-constraint certification converses} on explicit families, including a shared-action construction under a cross-moment condition.\\
Categorical mode identification \citep{jain2022mode} & Identification from categorical observations using KL information. & Comparison of the categorical reduction with the fuller off-policy experiment and its \textbf{recoverable failure labels}.\\\bottomrule
\end{tabularx}
\end{table}

\section{Experimental protocols and supporting results}\label{app:experiments_support}
This section documents the studies summarized in Section~\ref{sec:experiments}. The protocols come first, followed by implementation conditions, simulation references, and detailed results. The paired information study is in Appendix~\ref{app:infostudy}.

\subsection{Empirical protocols and results}\label{app:empirical_details}
Our experiments address five questions linking certification cost, information recovery, and deployment. \textbf{Q1}: does $W(1-p)\log(K/\delta)/\gamma^2$ describe when certification succeeds? \textbf{Q2}: does removing failure information cause the predicted cost, and does retaining counts or normalization totals remove it? \textbf{Q3}: is certifying a policy selected on the same data valid? \textbf{Q4}: does certification enable a better deployed policy than a safe default at realistic log sizes? \textbf{Q5}: does the selected policy meet its constraint on fresh held-out benchmark generations under the declared population model?

\paragraph{Certificates compared.}
All methods use $\delta=0.05$ and the same candidate lists; valid methods pay their applicable multiplicity and tuning costs. Exceedance methods are Exc-PB, Exc-Bet, Exc-EB, and Norm-Bet. Exc-EB uses the established CCDF construction of \citet[Lemma~4.4]{huang2022offarxiv} in the Maurer--Pontil form~\eqref{eq:summaryrepair}, unioned across candidates. Norm-Bet retains weight information (Appendix~\ref{app:sn}). The CDF methods are CDF-PB and CDF-EB; the latter is the per-threshold Thomas et al.\ interval used by UnO without truncation \citep{thomas2015high,chandak2021universal}. Split selects on one half-sample and certifies with Exc-PB on the other. The naive synthetic reference omits both search and tuning-grid penalties; the real-data PAC-Bayes and betting references omit only the search correction. These are invalid post-selection comparators, not interchangeable ablations.

\paragraph{Selection and data sharing.}
Full-sample outcome-table methods maximize importance-weighted utility among certified candidates, falling back to a reference policy and recording $\Psi=\bot$ if none certifies. This selection uses full logs even for CDF certificates. Full-sample methods in the rare-action boundary study select the tightest certified bound, breaking ties randomly. Its split baseline uses independently generated halves with the same total sample budget; the full-sample methods share counts. The paired information study instead uses identical logs for every method and selects the first certified candidate in a fixed order, keeping CDF-only decisions measurable in the stipulated reduction (Appendix~\ref{app:infostudy}).

\paragraph{Reading empirical log sizes.}
Reported $50\%$ and $90\%$ thresholds are estimates obtained by linear interpolation in log sample size between adjacent tested sizes at the first empirical rate crossing. They are not confidence bounds on the required sample size. A missing crossing is reported as ``--'' or a censored point. Grid spacing and repetition counts therefore limit precision, especially for the reasoning study's roughly twofold spacing.

\subsubsection{Q1--Q3: the rare-action class}
We simulate the class of Theorem~\ref{thm:boundary} with $W=64$, $K\in\{2,8,32\}$, $p\in\{0.90,0.95,0.99\}$ and $\gamma=(1-p)/2$, with $3000$ repetitions per point and sample sizes from $0.05\,n^\star$ to $3000\,n^\star$. We add an exact per-action Clopper--Pearson certificate as a model-specific reference, and an oracle ceiling on the average power of procedures valid on the full binary-outcome model (Appendix~\ref{app:ceiling}; simulation details in Appendix~\ref{app:simulation}).

\begin{figure}[!htbp]
\centering
\includegraphics[width=0.98\linewidth]{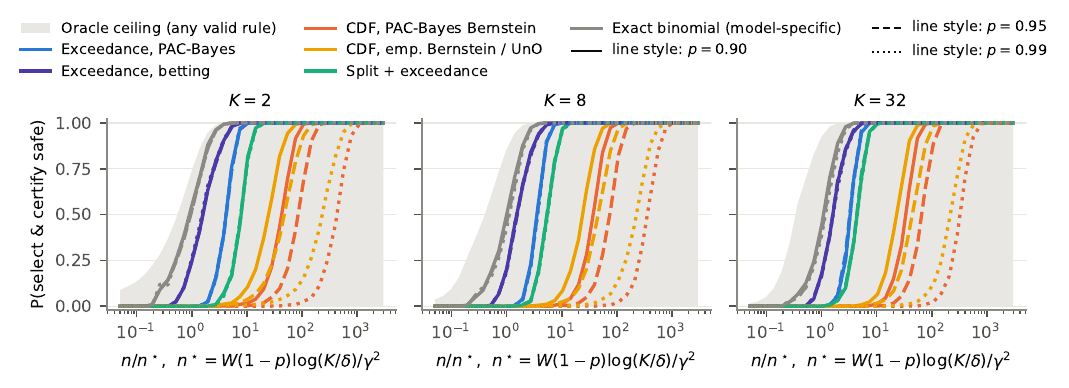}
\caption{\textbf{The tail-sharp boundary and the CDF penalty} (rare-action class, $W=64$, $K\in\{2,8,32\}$, $\delta=0.05$, margin $\gamma=(1-p)/2$; $3000$ repetitions per point; line style encodes $p$). Probability of selecting and certifying the unique safe candidate against $n$ divided by the predicted scale $n^\star=W(1-p)\log(K/\delta)/\gamma^2$. Exceedance certificates (PAC-Bayes in blue, betting in purple) nearly collapse across $p$, consistent with the scale in Theorem~\ref{thm:boundary}. Exc-PB has the grid-adjusted bound in Proposition~\ref{prop:pbpower}; betting rate attainment is empirical here. Relative to the PAC-Bayes exceedance certificate, the PAC-Bayes CDF envelope shifts right by $11$--$12\times$, $22$--$24\times$ and $110$--$121\times$ at $p=0.90$, $0.95$ and $0.99$, consistent with the order-$p/(1-p)$ separation (Theorem~\ref{thm:boundary}(b)). The shaded curve is the $p=0.90$ oracle Neyman--Pearson ceiling for procedures valid on the full binary-outcome model.}
\label{fig:boundary}
\end{figure}

\paragraph{Scaling and information loss.}
Figure~\ref{fig:boundary} shows the increasing CDF penalty across all tested $K$. Against Exc-Bet, the strongest tested CDF baseline (CDF-EB, the untruncated UnO per-threshold interval) needs $12$--$13\times$, $26\times$ and $122$--$135\times$ more data at $p=0.90$, $0.95$ and $0.99$, respectively.

\paragraph{Constants and reference procedures.} At $90\%$ power, Exc-Bet needs $n\approx(3.1$--$3.8)\,n^\star$, Exc-EB $(4.7$--$5.4)\,n^\star$ and Exc-PB $(5.0$--$6.7)\,n^\star$. Exc-Bet uses $1.4$--$1.8\times$ the log size of the model-specific exact-binomial reference and $2.1$--$3.5\times$ that of the oracle ceiling. Splitting requires $1.4$--$1.9\times$ the data of full-sample Exc-PB.

\paragraph{Validity.} Table~\ref{tab:validity} reports the null checks. Every generic valid procedure falsely certifies in at most $0.4\%$ of repetitions; the exact-binomial reference reaches $2.7\%$. The naive rule omitting both search and tuning penalties falsely certifies about $10\%$, twice the nominal $\delta$; this comparison does not isolate their individual effects.

\subsubsection{Q2--Q4: certifying the cost tail of LLM routers}
\paragraph{Data and candidates.} RouterBench \citep{hu2024routerbench} records the quality and dollar cost of five LLMs (GPT-4, GPT-3.5, Claude-2, Mixtral-8x7B, Mistral-7B) on $36{,}481$ prompts, split into $10{,}944$ development and $25{,}537$ evaluation prompts. On the development set, ridge predictors of quality and log-cost define nested families of $K\in\{8,32,128\}$ routers that trade quality against cost. The budget $\tau$ is chosen there for each $p\in\{0.90,0.95,0.99\}$ so that half of the $K=32$ routers are feasible with margin and the most accurate one is not. On the evaluation population, the best feasible router has quality $0.614$, the most accurate $0.779$ with $91$--$92\%$ of requests over budget, and the all-Mistral-7B fallback $0.305$.

\paragraph{Logging and repetitions.} Logging is $\pi_0^\epsilon=(1-\epsilon)\pi_{\rm base}+\epsilon/5$ with $\pi_{\rm base}$ the fallback and $\epsilon\in\{1,0.5,0.2\}$ ($W\in\{5,10,25\}$); the certifier sees only the logged model's cost. We use $300$ repetitions on $30$ log sizes for the main study, $500$ for stress tests.

\paragraph{Certification cost.} The largest gaps in Table~\ref{tab:router} are between CDF and exceedance certificates (curves in Appendix~\ref{app:routerfig}). Under uniform logging at $p=0.99$, CDF-EB does not reach $90\%$ certification within $300{,}000$ requests, versus $6{,}900$ for Exc-Bet. The CDF-EB/Exc-Bet log-size ratio shrinks to $1.8$--$3.7$ when logging is close to the certified routers ($\epsilon=0.2$). Thus the practical gap depends on logging as well as the tail level, consistent with the variance and range terms in Theorem~\ref{thm:instance}(a).

Differences among exceedance certificates are smaller: Exc-PB needs $1.50$--$1.94\times$ the logs of Exc-Bet, and Exc-EB $1.20$--$1.27\times$ ($1.56\times$ at $p=0.99$, $\epsilon=0.2$), about one grid step. Splitting requires $6.6$--$11\times$ the logs of full-sample Exc-PB; its certification rate temporarily plateaus near $0.4$ at intermediate log sizes because the selection half often picks a router just under the budget.

\paragraph{Explaining the gap from population moments (Q2).}
Figure~\ref{fig:predict} compares observed median log sizes with the empirical-Bernstein prediction obtained from evaluation-population variances and margins. Appendix~\ref{app:predict} gives the calculation and Table~\ref{tab:predict} its RouterBench values. No constants are fitted; this is a retrospective explanation, not an independently validated planning procedure. RouterBench predictions are within a factor of two for both statistics and within $10\%$ for exceedances, as close as the grid resolves. The gap follows the variance and range terms, rather than a fixed $p/(1-p)$ factor.

\begin{figure}[!htbp]
\centering
\begin{minipage}[c]{0.60\linewidth}
\includegraphics[width=0.9\linewidth]{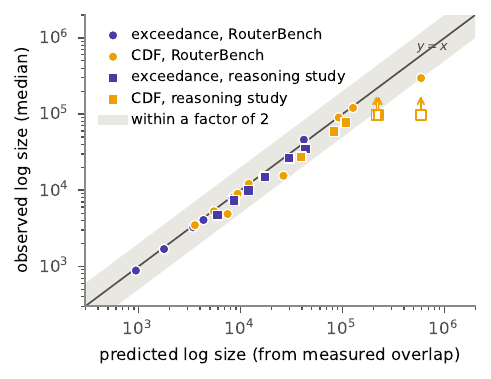}
\end{minipage}\hfill
\begin{minipage}[c]{0.4\linewidth}
\caption{\textbf{Retrospective explanation from population moments.} Log size needed for a $50\%$ certification rate, predicted from the measured tail-effective (exceedance) or bulk-effective (CDF) overlap with no fitted constants, against the observed value. Points cover both real-data studies, three logging policies and two or three quantile levels. Hollow markers with arrows: the CDF certificate did not reach a $50\%$ certification rate within the largest log tested, consistent with a prediction beyond it. The grey band marks a factor of $2$.}
\label{fig:predict}
\end{minipage}
\end{figure}

\textbf{\emph{Practical payoff (Q4).}} At $p=0.95$ and $\epsilon=0.5$, the router certified by the betting certificate from $3{,}640$ logged requests has quality $0.606$, close to the best feasible $0.614$. With the same logs, the CDF certificates still fall back to the Mistral-7B router (quality $0.305$). The empirical-Bernstein CDF bound reaches comparable quality only at about $100{,}000$ logs. The quality-greedy router would send about $91\%$ of requests over budget.

\textbf{\emph{Validity (Q3).}} No valid procedure, including the betting certificate, falsely certified in any cell of the main study. In two stress tests with up to $1024$ routers whose tail mass lies just above the budget, the worst rate was $0.8\%$ at nominal $5\%$, for Exc-Bet and Norm-Bet alike. Omitting the search cost is not safe. This is a winner's curse for tail certification: the same betting test without the union over candidates falsely certifies in $26\%$, $67\%$ and $100\%$ of repetitions with $32$, $128$ and $1024$ near-boundary routers, and a plug-in normal interval in every repetition.

\textbf{\emph{Designing logs for the tail.}} Appendix~\ref{app:tailaware} tests Proposition~\ref{prop:logging} for fixed stochastic routers, using development data to estimate tail probabilities. Tail-aware logging can reduce the required log size, but its benefit depends on the margin, range term, and certificate. Tables~\ref{tab:tailaware}--\ref{tab:tailaware_gr} report both gains and settings with little or no gain.

\subsubsection{Q2--Q5: reasoning budgets and held-out benchmark evaluation}
\paragraph{Data and outcomes.} We draw $16{,}000$ GSM8K~\citep{cobbe2021gsm8k} and MATH~\citep{hendrycks2021math} problems and split them uniformly at random, stratified by dataset, into $4{,}200$ development, $9{,}800$ evaluation and $2{,}000$ held-out benchmark problems. The held-out problems receive fresh generations only after router selection.

Seven open-weight configurations each answer the same $14{,}000$ development and evaluation problems once ($98{,}000$ generations): Qwen3-1.7B, Qwen3-8B and Qwen3-30B-A3B~\citep{qwen2025qwen3}, each with thinking off and on, and DeepSeek-R1-Distill-Qwen-7B~\citep{deepseek2025r1}. We use vLLM~\citep{kwon2023efficient} at the recommended settings and a cap of $8{,}192$ generated tokens. The outcome is generated tokens per request, a proxy for latency and cost, with constraint $P(\text{tokens}>\tau)\le1-p$. Thinking multiplies output length about $5\times$, with a heavy tail and an atom at the cap: $10$--$12\%$ of Qwen3 thinking traces ($6\%$ for R1-Distill) are truncated and almost always graded wrong (Table~\ref{tab:actions}).

\paragraph{Replay protocol.} Routers ($K\in\{8,32,128\}$), logging ($\epsilon\in\{1,0.5,0.2\}$, $W\in\{7,14,35\}$) and certificates follow RouterBench. The development rule gives $\tau=2{,}161$ tokens at $p=0.90$ and $3{,}509$ at $p=0.95$; at $p=0.99$ no budget separates the most accurate router from the safe reference, so we omit it. As in RouterBench, replay logs are drawn with replacement from the evaluation pool. These certificates and false-certification rates concern that finite pool, including when logs exceed its size.

\paragraph{Population model for the held-out study.} We additionally model each dataset's problems as an i.i.d.\ sample from a population of GSM8K-style or MATH-style requests. The target population is their mixture with the pool's GSM8K share $q=0.413$. Within each dataset, a uniformly random disjoint split of such a sample is again i.i.d. This assumption supports the separate population certificate described below.

\paragraph{Replay results.} The RouterBench pattern holds for generated-token outcomes (Table~\ref{tab:gpu}). At $p=0.90$ under uniform logging, Exc-Bet reaches $90\%$ certification from $4.3$k logs, against $46$k for CDF-EB and $86$k for CDF-PB. At $p=0.95$, neither tested CDF certificate reaches $90\%$ within $98$k requests under any logging policy (the best reaches $65\%$), whereas Exc-Bet needs $7.1$k, $12.6$k and $34$k. Exc-PB needs $1.7$--$2.2\times$ the logs of Exc-Bet and Exc-EB $1.7$--$2.1\times$. Norm-Bet is within $6\%$ of Exc-Bet: retaining the weight normalizer changes little here once exceedance information is used.

Valid procedures never falsely certified in the main study and did so in at most $0.4\%$ of stress-test repetitions. This stress family has few candidates near the boundary, so the naive certificate also rarely fails ($0.1\%$). Overlap-based predictions are slightly conservative: $13$--$27\%$ above observed log sizes for exceedances and about $40\%$ for the CDF. Where the CDF bound never reached a $50\%$ rate, the prediction ($215$k--$586$k) exceeds the largest log tested.

\paragraph{Held-out study: log and certificate (Q5).} At $p=0.90$ ($\tau=2{,}161$, $\epsilon=0.5$, $K=32$), we draw one log of $9{,}800$ requests in which every evaluation problem appears once, in random order, with its action drawn from the logging policy. Under the stated population model, observations are i.i.d.\ within each dataset; the split fixes dataset counts, so the log is not i.i.d.\ across datasets. We therefore use a \emph{stratified} betting certificate. The two datasets' betting wealths are independent, so their product at the true means is an e-value. The bound on $qG_{\rm GSM8K}+(1-q)G_{\rm MATH}$ uses a single confidence budget (Appendix~\ref{app:betting}) and concerns the declared population.

\paragraph{Certificate and held-out outcome.} Stratified Exc-Bet certifies a router that mostly uses Qwen3-30B-A3B without thinking, with bound $0.0968\le0.10$. The $2{,}000$ held-out problems receive fresh model generations; under the stated model they are a separate sample from the same dataset populations. They check the certificate empirically but are not part of its guarantee. Observed exceedance is $102/2000=0.051$ and accuracy $0.902$ (Table~\ref{tab:deploy}). This router is deterministic and coincides with the accuracy-greedy candidate on this log, whose fresh responses are reported.

\paragraph{Comparisons and scope.} Stratified Exc-PB does not certify from the same log, so it falls back to the reference router (accuracy $0.883$). From a $10\times$ larger log resampled with replacement, Exc-PB certifies a similar router (bound $0.062$, held-out exceedance $0.048$); this bound concerns the evaluation pool. Both evaluated certified routers had empirical held-out exceedance below budget. Always enabling reasoning with the largest model instead exceeds the budget on $53\%$ of fresh problems, with a $90$th percentile of $7{,}941$ tokens and accuracy $0.882$.

For comparison, the unstratified Exc-Bet diagnostic gives $0.0970$, while a union bound combining separate per-dataset bounds at $\delta/(2K)$ gives $0.111$ and certifies nothing. The unstratified Exc-PB diagnostic also has no threshold crossing. The reported population certificate uses the product-based stratified Exc-Bet bound above.

\begin{table}[!htbp]
\centering\small
\setlength{\tabcolsep}{2.5pt}
\caption{\textbf{Certified and always-reasoning routers have different empirical tail behavior.} Fresh generations on $2{,}000$ held-out benchmark problems ($p=0.90$, $\tau=2{,}161$ tokens, target exceedance $\le0.10$). The stratified bound concerns the stated population model; the $98$k bound concerns the resampled finite pool. Held-out measurements are empirical checks, not population guarantees.}
\label{tab:deploy}
\begin{tabular}{lrrrr}
\toprule
\rowcolor{tablehead}
\textbf{Router (log size)} &\textbf{Bound} &\textbf{\shortstack{Observed\\exceedance}} &\textbf{\shortstack{90th percentile\\(tokens)}} &\textbf{Accuracy} \\
\midrule
Exc-Bet certified, stratified ($9.8$k) & \keycell{$0.097$} & \keycell{$0.051$} & $1{,}435$ & \keycell{$0.902$}\\
Exc-PB on same log: fallback ($9.8$k) & -- & $0.053$ & $1{,}386$ & $0.883$\\
Exc-PB certified ($98$k, pool) & $0.062$ & $0.048$ & $1{,}362$ & $0.913$\\
\midrule
Always reason, 30B (uncertified) & -- & \keycell{$0.531$} & $7{,}941$ & $0.882$\\\bottomrule
\end{tabular}
\end{table}


\phantomsection\label{app:asset_records}
\paragraph{Asset provenance.} The code configuration records the model identifiers, sampling settings, and dataset loader candidates. The supplementary \texttt{ASSETS.md} gives primary-source attributions and stated licenses. The released prompt table records the source of every problem: GSM8K was loaded from \texttt{openai/gsm8k} and MATH from \texttt{EleutherAI/hendrycks\_math}, the first configured mirror. The RouterBench file used in the experiments has SHA-256 \texttt{ba4f77f1\ldots65d36d}, identical to the file committed by the publisher (\texttt{ASSETS.md}).

\phantomsection\label{app:compute_records}
\paragraph{Computing resources and reproducibility.}
Generation used one NVIDIA H100 NVL GPU (94\,GB) for about 17 GPU-hours, plus 1.2 GPU-hours for the held-out evaluation; simulations and replays used CPU instances with 13--24 effective cores, with run times in the supplementary READMEs. The inference software include vLLM, Transformers, PyTorch, CUDA. The supplement contains the prompts with their partitions and sources, the outcome table, and the per-request held-out generations (response text, extracted answer, grade and token count). Every held-out measurement in Table~\ref{tab:deploy} is recomputed exactly from these records; regenerating the responses requires a GPU and need not reproduce them token for token.

\phantomsection\label{app:data_content}
\paragraph{Data reuse and released content.}
The study reuses public benchmarks and recruits no participants; no separate provider/curator consent record is included. Released reasoning tables contain numerical outcomes, context features, problem identifiers, dataset/partition labels, and generation-status metadata; the released prompt table and held-out generation records also contain public benchmark problem text and model responses.

\subsection{Implementation details}\label{app:execution}
These implementation conditions complement the workflow in Section~\ref{sec:repair}; in particular, they specify when tuning and estimated propensities preserve the confidence guarantee.

\paragraph{Recovery computation and storage.}
Given the stored summaries, Algorithm~\ref{alg:recovered} uses $O(K|\mathcal T|)$ arithmetic operations and $O(K|\mathcal T|)$ storage including its inputs. The repair adds only $2K$ retained scalars. Accumulating these two weight totals from $n$ observations takes $O(nK)$ arithmetic operations, excluding policy-weight evaluation. These costs do not include constructing the original CDF summaries or utility scores, or learning the candidates.

\paragraph{Grids.} For grid-based certificates, $\lambda$ optimization uses a predeclared finite grid whose size is paid in the confidence term. The same applies to threshold grids $\mathcal T$.
\paragraph{Estimated propensities.}
For $\epsilon_{\rm prop}\in[0,1)$, suppose an independent sample yields a propensity estimate $\widehat\pi_0$ and a simultaneous lower bound $\pi_0\ge(1-\epsilon_{\rm prop})\widehat\pi_0$ with failure probability at most $\delta_{\rm prop}$. Let $\widehat w=\pi_\theta/\widehat\pi_0$ and a corresponding known bound $\widehat W$ be fixed from that sample. On its success event, $w\le\widehat w/(1-\epsilon_{\rm prop})$ and the true weights are bounded by $\widehat W/(1-\epsilon_{\rm prop})$. Conditional on this independently constructed design, apply Theorem~\ref{thm:exc} with error $\delta_{\rm cert}$, a predeclared admissible grid $\Lambda\subset(0,2(1-\epsilon_{\rm prop})/\widehat W)$, and that enlarged weight bound. The envelope is nondecreasing in its empirical exceedance and in the range bound, so replacing the unobserved true empirical exceedance by the plug-in estimate divided by $1-\epsilon_{\rm prop}$ is conservative. The final error probability is at most $\delta_{\rm prop}+\delta_{\rm cert}$; allocate their sum to be at most $\delta$. Unprotected plug-in propensities, or same-log propensity fitting without further analysis, are outside this guarantee.

\subsection{Oracle power ceiling}\label{app:ceiling}
This derives the oracle reference curve in Figure~\ref{fig:boundary}; the ceiling benchmarks achievable power but is not an executable certificate for deployment.

For a (randomized) selector-certifier, let $\phi_j=\PP(\Psi=j\mid\calD_n)$. We use two upper bounds on its average power $K^{-1}\sum_j\E_{P_j}\phi_j$ and report their minimum. For $q_0\in\{p-\gamma,p\}$ and $a\in[0,1]$, let $\mathcal R_n(a;q_0)$ denote the Neyman--Pearson power at size $a$ for testing ``action $j$ has $\PP(Y\le\tau)=q_0$'' against ``$=p+\gamma$'', all other actions being as in $P_j$. It is the same for all $j$ by symmetry.

\textbf{\emph{(a) Search-aware ceiling.}} Under the class null $P_0$ (Section~\ref{sec:boundary}), $P_0$ and $P_j$ differ only on action $j$. With $\alpha_j=\E_{P_0}\phi_j$ and $\sum_j\alpha_j\le\delta$, we have $\E_{P_j}\phi_j\le\mathcal R_n(\alpha_j;p-\gamma)$, and concavity and monotonicity of the randomized ROC give $K^{-1}\sum_j\E_{P_j}\phi_j\le\mathcal R_n(\delta/K;p-\gamma)$.

\textbf{\emph{(b) Boundary ceiling.}} Let $Q_j$ equal $P_j$ except that action $j$ has exceedance exactly $1-p$. $Q_j$ is a limit of instances in which every candidate is unsafe, and for fixed $n$ the law of $\calD_n$ is continuous in total variation, so validity gives $Q_j\{\Psi\ne\bot\}\le\delta$. Since $Q_j$ and $P_j$ differ only on $j$, $\E_{P_j}\phi_j\le\mathcal R_n(\delta;p)$.

Both bounds hold for procedures that are $\delta$-valid on the full binary-outcome model, including neighborhoods of the boundary; validity merely on the finite family in Theorem~\ref{thm:boundary} does not suffice for bound (b). Bound (b) ignores the search cost but uses the least-favourable null. We compute $\mathcal R_n$ exactly on the sufficient statistic (target-action count $N_j\sim\mathrm{Bin}(n,1/W)$, successes $S_j\mid N_j\sim\mathrm{Bin}(N_j,\cdot)$) by sorting states by likelihood ratio, randomizing on the boundary, and truncating at $9$ standard deviations. Omitted mass is credited to the ceiling, so truncation keeps it conservative.

\subsection{Simulation details}\label{app:simulation}
This records the sampling design and null checks behind the rare-action study in Section~\ref{sec:experiments}, including the reported simulation-code location.

\begin{table}[!htbp]
\centering\small
\caption{\textbf{Naive reuse exceeds nominal error in this stress test.} False-certification rate when all $K=32$ candidates are unsafe ($W=64$, nominal $\delta=0.05$). Each cell is the maximum over four sample sizes and over relative excess tail mass $\{2,5,20\}\%$, with $9000$ repetitions per configuration.}
\label{tab:validity}
\begin{tabular}{lccc}
\toprule
\rowcolor{tablehead}
\textbf{Procedure} &\textbf{$p=0.90$} &\textbf{$p=0.95$} &\textbf{$p=0.99$} \\
\midrule
Exc-PB (PAC-Bayes) & 0.0000 & 0.0000 & 0.0000\\
Exc-Bet (betting) & 0.0038 & 0.0027 & 0.0030\\
CDF, PAC-Bayes Bernstein & 0.0000 & 0.0001 & 0.0001\\
CDF, emp.\ Bernstein / UnO & 0.0002 & 0.0006 & 0.0012\\
Split & 0.0000 & 0.0000 & 0.0000\\
Exact binomial (reference) & 0.0266 & 0.0248 & 0.0238\\
\midrule
\keycell{Naive (no search/tuning penalties)} & \keycell{0.0972} & \keycell{0.0980} & \keycell{0.0963}\\\bottomrule
\end{tabular}
\end{table}

Here, per repetition, shared full-sample target-action counts are drawn from the exact multinomial and successes from binomials. The split reference independently generates halves of sizes $\lfloor n/2\rfloor$ and $\lceil n/2\rceil$; these have the usual i.i.d.\ split-sample law but are not paired with the full-sample draws. For $K\in\{2,8,32\}$, $W=64$, $p\in\{0.90,0.95,0.99\}$ and $\gamma=(1-p)/2$, sample sizes are $34$ geometric multiples of $n^\star$ in $[0.05,3000]$; $3000$ repetitions per cell. Data are generated in doubling blocks so that the betting certificate bets predictably. Validity: all $K=32$ candidates have exceedance $(1-p)(1+e)$ with $e\in\{0.02,0.05,0.2\}$, $9000$ repetitions. The exact binomial reference uses Clopper--Pearson lower limits at level $\delta/K$.

\subsection{Paired information study and shared actions}\label{app:infostudy}
\paragraph{Design.} The information study uses the family of Theorem~\ref{thm:boundary} with $\gamma=u/2$: $K=W=2$ for $p\in\{0.80,0.90,0.95,0.98,0.99,0.995\}$; $W=32$ with $K\in\{2,4,8,16,32\}$ at $p\in\{0.90,0.99\}$ and $K=8$ at $p\in\{0.80,0.95,0.98\}$; and one candidate logged with probability $1/W$, $W\in\{1,2,5,10,20,32\}$, at $p=0.95$. Each setting uses $44$ log sizes on a geometric grid from $0.02$ to $5000$ times $Wu\log(K/\delta)/\gamma^2$ and $2000$ repetitions per size; the safe index is drawn uniformly.

\paragraph{Count-based implementation.} Since $w_j=W\1\{A=j\}$, every certificate is an exact function of per-block counts of matched actions and of their exceedances, so the logs are generated as exact multinomial and binomial counts in the doubling blocks used by the betting certificates; a test in the supplementary code checks each count-based certificate against its observation-level implementation. Here $F_j:=1-G_j=\PP_{\pi_j}(Y\le\tau)$.

\paragraph{Rules and selection.} CDF-EB, CDF-Bet (the betting test applied to $1-w\1\{Y\le\tau\}/W$), and CDF-exact (Clopper--Pearson on $S_j\sim\mathrm{Bin}(n,F_j/W)$) see only the joint CDF indicators; Exc-EB and Exc-Bet see the exceedance indicators; Counts-exact (Clopper--Pearson on $N_j-S_j$ given $N_j$) and Norm-Bet see the action counts as well. Every rule uses level $\delta/K$ per candidate and outputs the first certified candidate in a fixed order. False certification is measured under the alternatives, under the all-unsafe family null ($u+\gamma$), and under a near-boundary null with every tail at $1.05u$.

\begin{table}[!htbp]
\centering\small
\setlength{\tabcolsep}{3pt}
\caption{\textbf{Compare exact CDF and count-based certificates.} Information study: log size (thousands) for $90\%$ selection-and-certification power, and the ratio CDF-exact/Counts-exact against $p/u$. Bold gives the CDF-exact/Counts-exact log-size ratio.}
\label{tab:infostudy}
\begin{tabular}{llrrrrrrrrr}
\toprule
& & \multicolumn{3}{c}{\cellcolor{cdfhead}\textbf{CDF only}} & \multicolumn{2}{c}{\cellcolor{exchead}\textbf{Exceedance}} & \multicolumn{2}{c}{\cellcolor{recoveryrow}\textbf{With counts}} & \multicolumn{2}{c}{\textbf{Ratios}}\\
\cmidrule(lr){3-5}\cmidrule(lr){6-7}\cmidrule(lr){8-9}\cmidrule(lr){10-11}
Setting & $p$ & EB & Bet & Exact & EB & Bet & Exact & Norm-Bet & CDF/counts & $p/u$\\
\midrule
$K=W=2$ & 0.8 & 2.2 & 2.2 & 1.0 & 0.8 & 0.5 & 0.3 & 0.5 & \keycell{3.6} & 4\\
$K=W=2$ & 0.9 & 8.0 & 9.1 & 4.2 & 1.6 & 1.1 & 0.6 & 1.1 & \keycell{6.9} & 9\\
$K=W=2$ & 0.95 & 30.2 & 36.6 & 16.7 & 3.1 & 2.2 & 1.3 & 2.2 & \keycell{12.7} & 19\\
$K=W=2$ & 0.98 & 190 & 227 & 110 & 7.8 & 5.5 & 3.2 & 5.4 & \keycell{33.9} & 49\\
$K=W=2$ & 0.99 & 740 & 922 & 421 & 15.6 & 11.2 & 6.8 & 11.3 & \keycell{61.5} & 99\\
$K=W=2$ & 0.995 & 2844 & 3675 & 1607 & 31.3 & 22.8 & 12.9 & 22.9 & \keycell{124.3} & 199\\
\midrule
$W=32$, $K=2$ & 0.9 & 222 & 270 & 122 & 24.9 & 18.6 & 9.8 & 16.9 & \keycell{12.3} & 9\\
$W=32$, $K=2$ & 0.99 & 22574 & 28157 & 12721 & 252 & 179 & 107 & 177 & \keycell{119.4} & 99\\
$W=32$, $K=4$ & 0.9 & 257 & 305 & 150 & 28.2 & 19.9 & 11.8 & 19.2 & \keycell{12.7} & 9\\
$W=32$, $K=4$ & 0.99 & 25229 & 32089 & 15267 & 288 & 209 & 128 & 207 & \keycell{119.7} & 99\\
$W=32$, $K=8$ & 0.9 & 286 & 346 & 167 & 30.6 & 21.7 & 13.5 & 20.7 & \keycell{12.4} & 9\\
$W=32$, $K=8$ & 0.99 & 28028 & 36019 & 18019 & 307 & 222 & 147 & 219 & \keycell{122.6} & 99\\
$W=32$, $K=16$ & 0.9 & 307 & 375 & 186 & 34.4 & 24.5 & 15.2 & 23.5 & \keycell{12.3} & 9\\
$W=32$, $K=16$ & 0.99 & 29646 & 40145 & 20259 & 345 & 245 & 163 & 241 & \keycell{124.0} & 99\\
$W=32$, $K=32$ & 0.9 & 322 & 403 & 212 & 38.3 & 25.7 & 16.5 & 25.0 & \keycell{12.8} & 9\\
$W=32$, $K=32$ & 0.99 & 32415 & 43760 & 21902 & 381 & 256 & 185 & 257 & \keycell{118.1} & 99\\
\midrule
one cand., $W=1$ & 0.95 & 1.4 & 1.0 & 0.5 & 1.4 & 1.0 & 0.5 & 1.0 & \keycell{1.0} & 19\\
one cand., $W=2$ & 0.95 & 27.6 & 31.3 & 14.5 & 2.8 & 2.0 & 1.1 & 1.9 & \keycell{13.7} & 19\\
one cand., $W=5$ & 0.95 & 103 & 123 & 55.1 & 6.9 & 5.1 & 2.7 & 5.0 & \keycell{20.6} & 19\\
one cand., $W=10$ & 0.95 & 234 & 273 & 115 & 14.1 & 9.8 & 5.4 & 9.3 & \keycell{21.6} & 19\\
one cand., $W=20$ & 0.95 & 493 & 564 & 259 & 28.5 & 20.0 & 10.5 & 19.1 & \keycell{24.6} & 19\\
one cand., $W=32$ & 0.95 & 806 & 932 & 402 & 45.0 & 32.2 & 17.5 & 31.5 & \keycell{23.0} & 19\\\bottomrule
\end{tabular}
\end{table}

\begin{table}[!htbp]
\centering\small
\setlength{\tabcolsep}{3pt}
\caption{\textbf{Maximum empirical rates across many configurations.} Information study: maximum false-certification rate over all settings and log sizes ($1100$ cells per null, $2000$ repetitions per cell, Monte Carlo standard error about $0.005$ at rate $0.05$). The exact tests are valid by construction; their rates average $0.016$ (CDF-exact) and $0.003$ (Counts-exact) over null cells.}
\label{tab:infovalidity}
\begin{tabular}{lrrr}
\toprule
\rowcolor{tablehead}
\textbf{Rule} &\textbf{family null} &\textbf{near boundary} &\textbf{alternatives} \\
\midrule
CDF-EB & 0.0025 & 0.0035 & 0.0010\\
CDF-Bet & 0.0025 & 0.0050 & 0.0010\\
CDF-exact & 0.0515 & 0.0640 & 0.0425\\
Exc-EB & 0.0000 & 0.0010 & 0.0000\\
Exc-Bet & 0.0010 & 0.0105 & 0.0000\\
Counts-exact & 0.0085 & 0.0390 & 0.0035\\
Norm-Bet & 0.0005 & 0.0090 & 0.0000\\\bottomrule
\end{tabular}
\end{table}

\paragraph{Shared actions.} The routers of Corollary~\ref{cor:affine} are ordered by a fixed random permutation, and the candidate list is its first $K$ entries, $K\in\{2,5,10,25,50,125\}$; the control list repeats the first five routers $25$ times each. With $p\in\{0.90,0.95\}$ and $\gamma=u/4$, the construction is checked exactly: agreement fractions are $0$ or $1/5$, $V_j^0=5(u+\gamma)\ge3\gamma W$, and every other router stays at least $0.6\gamma$ above the budget. Logs are cell counts over the $125$ (context, action) pairs, $36$ log sizes from $0.05$ to $400$ times $V\log(K/\delta)/\gamma^2$, where $V=V_j^0=5(u+\gamma)$, and $2000$ repetitions.

\begin{table}[!htbp]
\centering\small
\setlength{\tabcolsep}{5pt}
\caption{\textbf{The tested certificates need more data as the candidate list grows.} Shared actions: log size (thousands) for $90\%$ power with $K$ affine routers over five actions. Bold marks the endpoint comparisons; the repeated-candidate control retains the nominal multiplicity correction.}
\label{tab:shared}
\begin{tabular}{llrrrrrr}
\toprule
\rowcolor{tablehead}
\textbf{$p$} &\textbf{$K$} &\textbf{Exc-EB} &\textbf{Exc-Bet} &\textbf{Norm-Bet} &\textbf{Counts-exact} &\textbf{CDF-EB} &\textbf{CDF-exact} \\
\midrule
0.9 & \keycell{$2$} & \keycell{15.4} & 14.7 & 14.0 & 7.1 & \keycell{112} & 64.4\\
0.9 & $5$ & 17.7 & 17.5 & 16.6 & 8.7 & 131 & 79.3\\
0.9 & $10$ & 19.7 & 19.2 & 17.9 & 10.0 & 144 & 89.5\\
0.9 & $25$ & 22.1 & 21.9 & 20.5 & 11.3 & 160 & 103\\
0.9 & $50$ & 23.8 & 23.7 & 21.8 & 12.6 & 172 & 118\\
0.9 & \keycell{$125$} & \keycell{26.5} & 26.3 & 24.2 & 14.3 & \keycell{189} & 131\\
0.9 & \keycell{$125$ (5 distinct)} & 26.2 & 25.7 & 23.5 & 14.2 & 188 & 128\\\addlinespace[3pt]
0.95 & \keycell{$2$} & \keycell{31.3} & 30.3 & 28.5 & 14.0 & \keycell{458} & 263\\
0.95 & $5$ & 34.6 & 33.7 & 33.0 & 18.5 & 526 & 321\\
0.95 & $10$ & 39.5 & 39.0 & 37.9 & 21.2 & 574 & 369\\
0.95 & $25$ & 44.7 & 44.6 & 43.2 & 24.0 & 650 & 446\\
0.95 & $50$ & 48.5 & 47.7 & 45.3 & 26.4 & 693 & 493\\
0.95 & \keycell{$125$} & \keycell{52.1} & 51.7 & 50.1 & 29.4 & \keycell{757} & 529\\
0.95 & \keycell{$125$ (5 distinct)} & 53.4 & 52.8 & 50.5 & 30.6 & 759 & 544\\\bottomrule
\end{tabular}
\end{table}

\paragraph{Exact reconstruction.} On synthetic contextual logs with $20$ stochastic candidates (including zero weights, threshold ties and thresholds with no exceedance) and on both outcome tables with $16$ random stochastic targets under uniform logging, the exceedance moments reconstructed from CDF moments and the two totals of Section~\ref{sec:repair} agree with direct accumulation within $1.4\times10^{-11}$, the Exc-EB bounds to $10^{-14}$, and all certification decisions coincide.

\subsection{RouterBench curves}\label{app:routerfig}
Figure~\ref{fig:router} supplements the log-size summary in Table~\ref{tab:router} with the full RouterBench certification curves, showing how the comparison changes with logging overlap.

\begin{figure}[!htbp]
\centering
\includegraphics[width=0.9\linewidth]{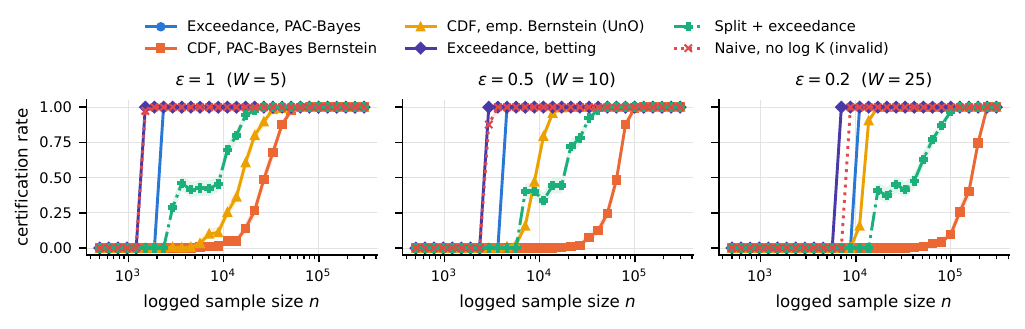}
\caption{\textbf{LLM routing (RouterBench, $p=0.95$, $K=32$).} Probability of certifying a router, as a function of the number of logged requests, for three logging policies. At $90\%$ empirical certification under uniform logging ($\epsilon=1$), Exc-Bet (purple) needs $18\times$ fewer logs than CDF-EB. Shading gives pointwise $95\%$ Wilson intervals over $300$ repetitions. The gap narrows as logging approaches the certified routers. The naive bound (red, dotted) omits the search correction and is invalid after selection.}
\label{fig:router}
\end{figure}
\FloatBarrier

\subsection{Predicted versus observed log sizes}\label{app:predict}
This expands the overlap interpretation in Section~\ref{sec:experiments} into a retrospective calculation; it explains the observed log sizes without treating population moments as an available planning input.

For router $j$ with exceedance $G_j<1-p$ and margin $\gamma_j=(1-p)-G_j$, an empirical-Bernstein certificate on a statistic with variance $\sigma^2$ certifies once $\sqrt{2\sigma^2L/n}+7WL/(3n)\le\gamma_j$, ignoring sampling fluctuation, where $L=\log(2K/\delta)$. For the exceedance statistic $\sigma^2=V_{\tau,j}-G_j^2$; for the CDF statistic $\sigma^2=V_{F,j}-(1-G_j)^2$. Both second moments are computed on the evaluation population for the given logging policy. This is a retrospective population-moment approximation, not a finite-sample power guarantee or an independent-pilot forecast. The prediction is the minimum of the resulting $n$ over feasible routers, compared with the observed log size at which the certification rate reaches $1/2$.

\begin{table}[!htbp]
\centering\small
\caption{\textbf{Observed thresholds versus retrospective predictions.} Predicted vs.\ observed median log sizes (RouterBench, $K=32$, fine grid). Bold values are observed thresholds; predictions use population moments.}
\label{tab:predict}
\begin{tabular}{llrrrr}
\toprule
\rowcolor{tablehead}
\textbf{$p$} &\textbf{$\epsilon$} &\textbf{Exc.\ pred.} &\textbf{Exc.\ obs.} &\textbf{CDF pred.} &\textbf{CDF obs.} \\
\midrule
0.90 & 1 & 940 & \keycell{881} & 7{,}464 & \keycell{4{,}870}\\
0.90 & 0.5 & 1{,}766 & \keycell{1{,}682} & 3{,}562 & \keycell{3{,}473}\\
0.90 & 0.2 & 4{,}327 & \keycell{4{,}065} & 5{,}473 & \keycell{5{,}285}\\
0.95 & 1 & 1{,}776 & \keycell{1{,}689} & 26{,}317 & \keycell{15{,}480}\\
0.95 & 0.5 & 3{,}432 & \keycell{3{,}261} & 9{,}359 & \keycell{8{,}979}\\
0.95 & 0.2 & 8{,}487 & \keycell{7{,}879} & 12{,}026 & \keycell{12{,}139}\\
0.99 & 1 & 8{,}347 & \keycell{7{,}879} & 588{,}945 & \keycell{296{,}815}\\
0.99 & 0.5 & 16{,}694 & \keycell{15{,}271} & 126{,}581 & \keycell{120{,}573}\\
0.99 & 0.2 & 41{,}735 & \keycell{46{,}008} & 91{,}739 & \keycell{89{,}995}\\\bottomrule
\end{tabular}
\end{table}

\subsection{Reasoning-budget study: outcomes and certificates}\label{app:gpu}
Tables~\ref{tab:actions} and~\ref{tab:gpu} supply the outcome characteristics and replay-certification results for the reasoning-budget study in Section~\ref{sec:experiments}.

\begin{table}[!htbp]
\centering\small
\setlength{\tabcolsep}{3pt}
\caption{\textbf{Reasoning configurations have an atom at the token cap.} Outcome table: accuracy, generated tokens and truncation rate at the $8{,}192$-token cap ($14{,}000$ problems per configuration).}
\label{tab:actions}
\begin{tabular}{lrrrr}
\toprule
\rowcolor{tablehead}
\textbf{Configuration} &\textbf{Accuracy} &\textbf{Mean tokens} &\textbf{95th pct.\ tokens} &\textbf{Truncated} \\
\midrule
Qwen3-1.7B & 0.787 & 688 & 2{,}276 & 0.7\%\\
Qwen3-1.7B, thinking & 0.833 & 3{,}446 & \keycell{8{,}192} & \keycell{12.5\%}\\
Qwen3-8B & 0.892 & 707 & 2{,}196 & 1.0\%\\
Qwen3-8B, thinking & 0.867 & 3{,}528 & \keycell{8{,}192} & \keycell{11.2\%}\\
Qwen3-30B-A3B & 0.908 & 667 & 2{,}205 & 0.5\%\\
Qwen3-30B-A3B, thinking & 0.882 & 3{,}207 & \keycell{8{,}192} & \keycell{9.6\%}\\
R1-Distill-Qwen-7B & 0.904 & 2{,}540 & \keycell{8{,}192} & \keycell{5.7\%}\\\bottomrule
\end{tabular}
\end{table}

\begin{table}[!htbp]
\centering\small
\setlength{\tabcolsep}{3pt}
\caption{\textbf{The EB routes separate again on token budgets.} Estimated log sizes (thousands) at $90\%$ empirical certification of the selected router's token budget (seven open-weight configurations, $K=32$, $\delta=0.05$, $1000$ repetitions). ``--'': not reached within $98$k. $^\dagger$The empirical rate first crosses $0.90$ at the largest size ($904/1000$; $95\%$ Wilson interval $[0.884,0.921]$). Log sizes form a coarse grid of about $\times2$ steps; logs larger than the $9{,}800$-problem evaluation pool resample it with replacement. Bold compares Exc-EB with CDF-EB.}
\label{tab:gpu}
\begin{tabular}{llrrrrrrr}
\toprule
\rowcolor{tablehead}
\textbf{$p$} &\textbf{$\epsilon$} &\textbf{Exc-PB} &\textbf{Exc-Bet} &\textbf{Exc-EB} &\textbf{Norm-Bet} &\textbf{CDF-EB} &\textbf{CDF-PB} &\textbf{Split} \\
\midrule
0.90 & 1 & 8.7 & 4.3 & \keycell{8.5} & 4.3 & \keycell{46.2} & 86.2 & 18.2\\
0.90 & 0.5 & 18.9 & 8.8 & \keycell{18.4} & 8.6 & \keycell{97.5$^\dagger$} & -- & 36.3\\
0.90 & 0.2 & 41.7 & 20.7 & \keycell{41.5} & 20.4 & \keycell{--} & -- & 85.7\\\addlinespace[3pt]
0.95 & 1 & 13.6 & 7.1 & \keycell{11.8} & 6.8 & \keycell{--} & -- & 21.4\\
0.95 & 0.5 & 21.7 & 12.6 & \keycell{21.2} & 12.7 & \keycell{--} & -- & 45.9\\
0.95 & 0.2 & 74.9 & 34.4 & \keycell{68.1} & 35.5 & \keycell{--} & -- & --\\\bottomrule
\end{tabular}
\end{table}

\subsection{Alternative exceedance bounds}\label{app:klvariant}
This supporting comparison holds the exceedance statistic fixed to distinguish the choice of confidence bound from the information-loss mechanism in Section~\ref{sec:boundary}.

We compare three bounds on the same weighted exceedance statistic, using the rare-action design of Section~\ref{sec:experiments} ($W=64$, $\gamma=(1-p)/2$, $3000$ repetitions). The bounds are the PAC-Bayes Bernstein bound of Theorem~\ref{thm:exc}, the empirical-Bernstein bound, and the PAC-Bayes-kl bound $\kl(\widehat G_\rho(\tau)/W\,\|\,\bar G_\rho(\tau)/W)\le\{\KL(\rho\|\pi)+\ln(2\sqrt n/\delta)\}/n$ \citep{seeger2002pac,maurer2004note}, applied to the $[0,1]$-valued loss $w\1\{Y>\tau\}/W$. In this finite-list experiment, $\rho=\delta_{\theta_j}$ selects candidate $j$ and $\pi$ is uniform on the $K$ candidates, so $\KL(\rho\|\pi)=\log K$. Table~\ref{tab:klvariant} reports $n/n^\star$ at $90\%$ power from a separate run (different seed and grid; its PAC-Bayes values are within $7\%$ of the main run). The kl bound is the loosest of the three in every setting. This comparison shows only that this particular PAC-Bayes-kl implementation does not close the gap. It does not isolate its cause: tuning penalties, finite-sample constants, and conditioning on action counts all differ across methods. In particular, the exact conditional-binomial reference uses normalization information unavailable from the exceedance sequence alone.

\begin{table}[!htbp]
\centering\small
\caption{\textbf{This PAC-Bayes-kl implementation does not reduce log size.} $n/n^\star$ at $90\%$ selection-and-certification power on the rare-action class. Bold marks the tested alternative, not a winner.}
\label{tab:klvariant}
\begin{tabular}{llrrrr}
\toprule
\rowcolor{tablehead}
\textbf{$K$} &\textbf{$p$} &\textbf{Exact binomial} &\textbf{Exc-EB} &\textbf{Exc-PB} &\textbf{Exc-kl} \\
\midrule
2 & 0.90 & 2.06 & 5.16 & 6.60 & \keycell{6.84}\\
2 & 0.95 & 2.24 & 5.14 & 6.52 & \keycell{7.12}\\
2 & 0.99 & 2.24 & 5.14 & 6.63 & \keycell{7.63}\\
8 & 0.90 & 2.07 & 4.78 & 5.42 & \keycell{5.56}\\
8 & 0.95 & 2.17 & 4.73 & 5.38 & \keycell{5.58}\\
8 & 0.99 & 2.25 & 4.78 & 5.49 & \keycell{5.98}\\
32 & 0.90 & 2.05 & 4.43 & 4.72 & \keycell{4.89}\\
32 & 0.95 & 2.11 & 4.42 & 4.71 & \keycell{4.94}\\
32 & 0.99 & 2.21 & 4.43 & 4.69 & \keycell{5.20}\\\bottomrule
\end{tabular}
\end{table}

\subsection{Tail-aware logging}\label{app:tailaware}
These fixed-target experiments illustrate Proposition~\ref{prop:logging} and complement Section~\ref{sec:experiments} by showing when a smaller tail second moment translates into fewer logs.

\paragraph{RouterBench.}
The target routers are $\pi_b(a\mid x)=(1-b)\1\{a=\text{Mistral-7B}\}+b\1\{a=\text{GPT-4}\}$, with $b$ set so that $G(\tau)=(1-r)(1-p)$ on the evaluation population for relative margins $r\in\{0.1,0.2,0.5\}$, using the budgets $\tau$ of the main study. The estimate $\hat g$ is the per-(task, model) exceedance frequency on the development split, shrunk to the per-model mean and floored at $10^{-3}$. The target-policy and tail-aware logs mix in $10\%$ uniform exploration. For each policy and log size (32 sizes from $200$ to $400{,}000$, $300$ repetitions), we draw logs from the evaluation population and certify the fixed target ($K=1$, $\delta=0.05$). Table~\ref{tab:tailaware} reports estimated log sizes (thousands) at $90\%$ empirical certification. Gains depend on the margin and method: at a $50\%$ relative margin the range term leaves little benefit, and Exc-PB can require more data because its worst-case proxy $WG$ increases with $W$ despite a smaller tail second moment.

\begin{table}[!htbp]
\centering\small
\setlength{\tabcolsep}{3pt}
\caption{\textbf{Logging gains depend on the certificate and margin.} Tail-aware vs.\ target-policy (with $10\%$ uniform exploration floor) vs.\ uniform logging: logs (thousands) to certify a fixed stochastic router with probability $0.9$; margin is the relative margin $r=\gamma/(1-p)$, and the last column is target-policy$/$tail-aware. Bold ratios above one favor tail-aware logging; ratios below one favor target-policy logging.}
\label{tab:tailaware}
\begin{tabular}{lllrrrr}
\toprule
\rowcolor{tablehead}
\textbf{$p$} &\textbf{margin} &\textbf{certificate} &\textbf{target} &\textbf{tail-aware} &\textbf{uniform} &\textbf{ratio} \\
\midrule
0.90 & 10\% & Exc-Bet & 12.7 & 3.5 & 7.8 & \keycell{3.7}\\
0.90 & 10\% & Exc-EB & 12.5 & 6.4 & 13.5 & \keycell{1.9}\\
0.90 & 10\% & Exc-PB & 25.4 & 66.8 & 81.8 & \keycell{0.4}\\
0.90 & 20\% & Exc-Bet & 3.3 & 1.4 & 2.5 & \keycell{2.3}\\
0.90 & 20\% & Exc-EB & 3.5 & 2.6 & 4.6 & \keycell{1.4}\\
0.90 & 20\% & Exc-PB & 6.5 & 15.6 & 20.4 & \keycell{0.4}\\
0.90 & 50\% & Exc-Bet & 0.4 & 0.4 & 0.8 & \keycell{0.9}\\
0.90 & 50\% & Exc-EB & 0.6 & 0.7 & 1.3 & \keycell{0.9}\\
0.90 & 50\% & Exc-PB & 0.9 & 1.8 & 3.3 & \keycell{0.5}\\
0.95 & 10\% & Exc-Bet & 27.0 & 4.5 & 10.0 & \keycell{6.0}\\
0.95 & 10\% & Exc-EB & 25.9 & 7.9 & 19.4 & \keycell{3.3}\\
0.95 & 10\% & Exc-PB & 49.6 & 87.7 & 165.9 & \keycell{0.6}\\
0.95 & 20\% & Exc-Bet & 6.3 & 1.8 & 3.8 & \keycell{3.5}\\
0.95 & 20\% & Exc-EB & 6.9 & 3.2 & 7.5 & \keycell{2.2}\\
0.95 & 20\% & Exc-PB & 11.9 & 20.4 & 42.1 & \keycell{0.6}\\
0.95 & 50\% & Exc-Bet & 0.6 & 0.5 & 1.4 & \keycell{1.1}\\
0.95 & 50\% & Exc-EB & 1.1 & 0.9 & 2.3 & \keycell{1.3}\\
0.95 & 50\% & Exc-PB & 1.7 & 2.4 & 6.1 & \keycell{0.7}\\\bottomrule
\end{tabular}
\end{table}

\paragraph{Reasoning study.} Same protocol on the outcome table of Section~\ref{sec:experiments} (outcome: generated tokens; budgets $\tau=2{,}161$ and $3{,}509$ tokens from the main study). The targets answer with Qwen3-30B-A3B without thinking and escalate a fraction $b$ of requests to the same model with thinking, with $b$ set so that $G(\tau)=(1-r)(1-p)$ on the evaluation population ($p=0.90$, $r=0.5$ is infeasible because the non-thinking model alone exceeds $0.05$). The estimate $\hat g$ is the development-split exceedance frequency per (dataset, difficulty level, configuration), shrunk to the per-configuration mean; $300$ repetitions per log size. Because the target puts little weight on the risky action, the across-action variance of $\sqrt g$, which is the gain in Proposition~\ref{prop:logging}, is small relative to $V_\tau$, so a small gain is predicted. The observed log-size ratio tracks the $V_\tau$ ratio (Table~\ref{tab:tailaware_gr}).

\begin{table}[!htbp]
\centering\small
\setlength{\tabcolsep}{2.5pt}
\caption{\textbf{Gains are modest and method-dependent here.} Tail-aware logging on the reasoning study: relative margin $r=\gamma/(1-p)$, escalation rate $b$, $V_\tau$ ratio target-policy$/$tail-aware, logs (thousands) for Exc-Bet to certify with probability $0.9$ under target-policy, tail-aware and uniform logging, and the target$/$tail-aware log-size ratio for Exc-Bet and Exc-EB. Bold compares tail second-moment and log-size ratios, including values below one.}
\label{tab:tailaware_gr}
\begin{tabular}{llrrrrrrr}
\toprule
\rowcolor{tablehead}
\textbf{$p$} &\textbf{margin} &\textbf{$b$} &\textbf{$V_\tau$ ratio} &\textbf{target} &\textbf{tail-aware} &\textbf{uniform} &\textbf{ratio (Bet)} &\textbf{ratio (EB)} \\
\midrule
0.90 & 30\% & 0.038 & \keycell{1.11} & 1.5 & 1.4 & 8.6 & \keycell{1.13} & \keycell{0.89}\\
0.90 & 20\% & 0.059 & \keycell{1.18} & 3.6 & 3.2 & 18.1 & \keycell{1.14} & \keycell{0.95}\\
0.90 & 10\% & 0.079 & \keycell{1.25} & 17.1 & 12.8 & 65.2 & \keycell{1.33} & \keycell{1.03}\\
0.95 & 50\% & 0.011 & \keycell{1.01} & 1.0 & 0.9 & 5.7 & \keycell{1.08} & \keycell{0.99}\\
0.95 & 30\% & 0.044 & \keycell{1.18} & 3.3 & 2.7 & 14.7 & \keycell{1.22} & \keycell{0.94}\\
0.95 & 20\% & 0.060 & \keycell{1.26} & 7.5 & 6.2 & 30.9 & \keycell{1.21} & \keycell{1.00}\\
0.95 & 10\% & 0.076 & \keycell{1.33} & 34.2 & 27.8 & 114.0 & \keycell{1.23} & \keycell{1.13}\\\bottomrule
\end{tabular}
\end{table}

\end{document}